\documentclass{article}
\usepackage[preprint]{neurips_2023}
\usepackage[utf8]{inputenc}
\usepackage[T1]{fontenc}
\usepackage{microtype}
\usepackage{amsmath,amssymb,amsthm,mathtools}
\usepackage{booktabs,multirow,array}
\usepackage{graphicx}
\usepackage{xcolor}
\usepackage[capitalize,noabbrev]{cleveref}

\graphicspath{{figures/}}

\newtheorem{proposition}{Proposition}

\theoremstyle{definition}\newtheorem{definition}{Definition}
\theoremstyle{remark}\newtheorem{remark}{Remark}

\newcommand{\cheap}{\mu}            
\newcommand{\strong}{y}             
\newcommand{\human}{h}              
\newcommand{\anchorset}{\mathcal{A}}
\newcommand{\gapstat}{\tilde{g}}
\newcommand{\gapbase}{\tilde{q}_0}
\newcommand{\guard}{W}
\newcommand{\wealth}{\mathcal{W}}   
\newcommand{\alphasys}{\alpha_{\mathrm{sys}}}
\newcommand{\alphajudge}{\alpha_{\mathrm{judge}}}

\title{Who Drifted: the System or the Judge?\\Anytime-Valid Attribution in LLM Evaluation Pipelines}
\author{Yitao Li\\\texttt{yitaoli416@gmail.com}}

\begin{document}
\maketitle
\begin{abstract}
Continuous evaluation of LLM products relies on a strong LLM judge treated as ground
truth: a cheap monitor scores every interaction and a team is paged when the score drifts
down. But the judge is itself a model behind an API, and a silent version bump or
scoring-prompt update changes how it scores --- so every drift alarm is ambiguous between
a worse product and a changed judge. We resolve the ambiguity with a fixed, human-labeled
\emph{anchor set} that the \emph{current} judge re-scores at a steady interleave, a second
betting e-process on the judge-versus-human gap, and a guard-window rule returning a
verdict in $\{\texttt{none}, \texttt{system}, \texttt{judge}\}$. We prove anytime-validity,
one-way identification (only the judge can move the anchors), an attribution race whose
design law is that the anchors must out-run the main process they guard, and process
orthogonality. On two \emph{real} judge changes, a silent version bump is detected as
judge drift in $60/60$ runs with zero judge-to-system misattribution, and a contaminating
strict-prompt change is correctly attributed on $110$ of $120$ runs at guard width $300$
--- while the industry-default rolling $z$-test false-alarms on $75\%$ of drift-free
streams. Every experiment replicates on a second domain (TL;DR summarization) with nothing
re-tuned, and where the domains differ the differences are the ones the race predicts: the
strict-prompt change shifts scores harder there, so the anchors fire faster and attribution
becomes perfect ($240/240$). The monitor runs at $\approx 0.64$ of the cost of strong-judging
every item, or $0.21$ in a cheaper-but-deafer regime.
\end{abstract}

\section{Introduction}\label{sec:intro}


Large language model products are now evaluated continuously by other large language models. An
LLM-as-judge \citep{zheng2023judge} scores each interaction, a dashboard tracks the rolling mean, and
an on-call team is paged when the score drifts down. The implicit contract is that the judge is a
fixed yardstick, so a falling score means a falling product. But the strong judge is itself a model
behind an API: a version bump or a silent prompt-and-policy update can change how it scores, with no
announcement and no change to the product at all. When the monitor fires, the team faces a question
the dashboard cannot answer --- \emph{which thing drifted, the system or the judge?} Acting on the
wrong answer is expensive in both directions: rolling back a healthy product, or shipping a real
regression because the alarm was dismissed as ``just the judge.''

This is not a hypothetical edge case; it is a property of the tools already in use. Consider the
de-facto industry practice for watching a judge: re-test the judge-versus-human gap on a rolling
window with a $z$-test at $\alpha = 0.05$, every time a new observation arrives. Run on the same
drift-free stream our anchor process consumes, that procedure false-alarms on $75\%$ of streams where
nothing ever changes, and $16$ of $60$ of its ``detections'' fire \emph{before the change point even
exists} (\cref{sec:experiments}). Its apparent vigilance is the absence of error control, not
sensitivity: a fixed-$\alpha$ test re-run at every step has no validity guarantee at a
data-dependent stopping time. Continuous monitoring demands \emph{anytime-valid} tooling --- e-values
and test supermartingales that may be inspected at every step without inflating the false-alarm rate
\citep{ville1939, waudbysmith2024betting, howard2021timeuniform}. Once those are in place, the
detector reliably tells you \emph{that} something drifted. \emph{Attribution} --- whether it was the
system or the judge --- is the remaining hole, and it is the one we fill.

\paragraph{The construction.} We start from the prediction-powered monitor a cost-conscious team
would deploy anyway: a cheap judge $\cheap$ scored on \emph{every} item, the expensive strong judge
$\strong$ sampled under a budget, and a per-stratum prediction-powered e-process
\citep{csillag2025ppe, angelopoulos2023ppi} that accrues anytime-valid evidence of system drift
(\cref{sec:setting}). To this we add one component: a fixed, human-labeled \emph{anchor set}
$\anchorset$, held out before monitoring begins and re-scored by the \emph{current} judge at a steady
interleave, driving a second e-process that watches the judge's gap to its frozen human labels
(\cref{sec:anchor}). A guard-window rule combines the two stopping times into a verdict in
$\{\texttt{none}, \texttt{system}, \texttt{judge}\}$. The construction turns on a deliberate
asymmetry. The anchors are frozen, so the product's drift can never reach them --- improve the system
or wreck it, the gap to human labels on those fixed items does not move. Only a change in the
\emph{judge} can move it. The anchor process is therefore a one-way mirror: it sees judge drift, and
nothing else.

\paragraph{Contributions.}
\begin{enumerate}
\item \textbf{The anchor construction and guard-window attribution rule}
  (\cref{sec:anchor})\,---\,a human-anchored, anytime-valid e-process that disentangles judge drift
  from system drift, emitting an explicit $\{\texttt{none}, \texttt{system}, \texttt{judge}\}$
  verdict rather than an undifferentiated alarm.
\item \textbf{Four propositions} (\cref{sec:anchor}): anytime-validity of the anchor family,
  one-way identification (only the judge can move the anchors), the attribution race, and process
  orthogonality. The race proposition exposes a design law --- the anchor process must be provisioned
  to \emph{out-run} the main process it guards, since a contaminated reference can otherwise fire a
  spurious system alarm before the anchors catch up.
\item \textbf{Two \emph{real} judge changes}, not only synthetic drift (\cref{sec:experiments}): a
  silent version bump (\texttt{gemini-3.1-pro} $\to$ \texttt{3.5-flash}) detected as judge drift in
  $60/60$ runs with $0$ judge-to-system misattribution at the featured configuration, and a
  strict-prompt policy change that genuinely contaminates the main monitor yet is correctly
  attributed on $110$ of $120$ contaminated runs at guard width $\guard = 300$ --- both replicated
  exhibit-for-exhibit on a second domain (TL;DR summarization) with nothing re-tuned. The
  replication is a test of the theory, not only of robustness: the invariances transfer (the same
  version bump leaves the same $+0.03$ to $+0.07$ signature in both domains), and where the domains
  differ, the differences are the ones the propositions predict --- the strict-prompt shift is
  larger on summaries, so the anchors fire faster, attribution becomes perfect ($240/240$), and the
  co-provisioning band moves toward slower anchor rates, exactly the attribution race.
\item \textbf{Operational evidence that classical alternatives fail} (\cref{sec:experiments}): on the
  identical anchor stream, naive repeated $z$-testing false-alarms on $75\%$ of drift-free streams,
  and a carefully calibrated Page--Hinkley detector holds its false-alarm budget but detects the real
  lenient bump only $8\%$ of the time and requires held-out calibration streams the e-process does
  not.
\item \textbf{The cost-aware foundation} (\cref{sec:setting}): a stratified prediction-powered
  monitor that detects localized blind-spot regressions at roughly one third the cost of strong-evaluating
  every item, with an e-wealth escalation trigger that is far cheaper than covariate-triggered prior
  rules (which oversample to near full-evaluation cost or read inert under a biased cheap judge),
  while merely tying a matched fixed budget on the faintest drift.
\end{enumerate}

\Cref{sec:related} situates these contributions; \cref{sec:setting} builds the budgeted main monitor;
\cref{sec:anchor} adds the anchor construction and its guarantees; and \cref{sec:experiments}
validates all of them on two datasets and two real judge changes before \cref{sec:discussion}
discusses limitations.

\section{Related Work}\label{sec:related}


\paragraph{Prediction-powered inference and anytime-valid testing.} Prediction-powered inference
\citep{angelopoulos2023ppi, angelopoulos2023ppiplusplus} uses a cheap predictor to debias a small set
of expensive labels, and has been extended in many directions: cross-fitting
\citep{zrnic2024crossppi}, frequentist--assisted--Bayes constructions \citep{kilian2025bayesppi},
federated-and-bagged variants \citep{cortinovis2025fabppi}, conformal procedures
\citep{csillag2025conformal}, and power analysis \citep{chen2026power}. Our main monitor builds on the
\emph{prediction-powered e-value} of \citet{csillag2025ppe}, whose Thm.~2.1 our per-stratum process
inherits, and sits in the broader anytime-valid tradition: test (super)martingales and Ville's
inequality \citep{ville1939}, time-uniform concentration \citep{howard2020chernoff,
howard2021timeuniform}, betting confidence sequences \citep{waudbysmith2024betting}, e-detectors for
sequential change \citep{shin2024edetectors}, sequential comparison of forecasters
\citep{choe2021forecasters}, and prediction-powered \emph{risk} monitoring of deployed models
\citep{zhang2026ppmonitoring}. We claim no new validity theory here: the combination of a
cheap-proxy PPI estimator, an anytime-valid e-process, and drift detection already exists in this
cluster. Our \cref{sec:setting} contribution is the specific assembly --- a per-(rubric$\times$stratum)
Bonferroni e-process for \emph{localized blind-spot} drift, plus an e-wealth escalation trigger ---
benchmarked against this cluster on a single cost-versus-latency frontier. The genuinely new material
is the anchor construction and attribution analysis of \cref{sec:anchor}.

\paragraph{Active and adaptive acquisition.} Several lines couple a sampling rule to a downstream
inference or detection target: active inference that labels where uncertainty is highest
\citep{zrnic2024active}, the active-acquisition rule in Csillag et al.'s Appendix~B.2 (which we
include as a baseline), round-robin active change detection \citep{chaudhuri2024roundrobin},
adaptive-sampling change detection \citep{yi2025adaptive}, and drift-to-action controllers that pair a
detector with an acquisition policy \citep{lamaakal2026drift2act}. A cautionary result frames our
escalation finding: \citet{sfyraki2026adaptive} prove that for prediction-powered mean estimation,
uncertainty-driven adaptive sampling is asymptotically no better than a fixed budget. Our result is
\emph{consistent} with theirs --- our e-wealth trigger merely ties a matched fixed budget on the
faintest drift --- and our claim is narrowly about cost-efficiency relative to other adaptive
\emph{triggers} (covariate-driven rules waste budget or read inert under a biased cheap judge), not
asymptotic dominance over fixed-$\tau$.

\paragraph{LLM-as-judge reliability.} A growing literature studies the reliability of LLM judges
\citep{zheng2023judge}: correctly reporting results obtained with an imperfect judge
\citep{lee2026reporting}, valid downstream inference under noisy judge labels
\citep{feng2026noisy}, and quantifying judge bias and uncertainty \citep{fiedler2026bias}. That work
corrects or qualifies a \emph{static} judge --- it treats the judge's errors as a fixed, if unknown,
distortion to be debiased. We address a different failure mode: a judge whose behavior \emph{changes
over time}, and the problem of separating those changes from changes in the system under
evaluation. To our knowledge, no prior work runs a human-anchored e-process specifically to detect
and attribute drift in the judge's identity.

\section{Cost-Aware Monitoring with a Budgeted Strong Judge}\label{sec:setting}


\subsection{Stream model}
A deployment emits items $i = 1, 2, \dots$ in a stream. Each item is scored on $R$ rubrics, with
every score normalized to $[0,1]$, and carries an observable stratum label $t \in \{1, \dots, K\}$
(here a topic). Two judges are available. A \emph{cheap} judge $\cheap$ is run on \emph{every} item
and returns a vector $\cheap_i \in [0,1]^R$; a \emph{strong} judge $\strong$ returns
$\strong_i \in [0,1]^R$ but is expensive and is therefore queried only on a sampled subset of items,
under a budget. In this section $\strong$ is the trusted reference: a drop in $\strong$ is, by
assumption, a real drop in quality.

\subsection{A per-stratum prediction-powered e-process}
The monitor maintains one e-process per (rubric $\times$ stratum) cell, testing the one-sided null
that the cell mean has not fallen below a healthy bar,
\begin{equation}
  H_0^{\mathrm{sys}}:\quad \mathbb{E}[\strong_i[r] \mid i \in \text{cell } (r,t)] \;\geq\; q_0[r,t] - 0.05,
  \label{eq:sys-null}
\end{equation}
where the per-cell bar $q_0[r,t]$ is the mean strong-judge score over in-control items of stratum
$t$, estimated once from held-out calibration data, and the $0.05$ slack absorbs calibration noise.
Each cell accumulates wealth through prediction-powered e-values \citep{csillag2025ppe}: the cheap
score $\cheap_i[r]$ supplies a debiasing prediction, and on sampled items the strong label corrects
it, so a small number of strong queries yields a valid e-value with the cheap judge as free side
information \citep{angelopoulos2023ppi}. The bets are \emph{predictable} --- chosen from the past ---
and capped at an effective $\lambda_{\max}(\pi_i)$ that depends only on the (predictable) sampling
probability $\pi_i$, which keeps each cell's wealth a nonnegative supermartingale under
\cref{eq:sys-null}. An alarm fires when any cell's wealth crosses the Bonferroni threshold
$R K / \alphasys$. Validity is then inherited from \citet[Thm.~2.1]{csillag2025ppe} together with the
predictability of $\pi$ and Ville's inequality \citep{ville1939}: the family's probability of ever
falsely alarming, at any data-dependent stopping time, is at most $\alphasys$.

The per-stratum design is not a refinement but a necessity. A localized blind-spot drift --- a
quality regression confined to one stratum, on a rubric where the cheap judge is blind --- is diluted
to invisibility when pooled. Empirically, under the mild drift a single \emph{global} per-rubric PPI
e-process never detects at any budget (censoring rate $1.00$ at strong-sampling rates of $0.1$, $0.2$,
and $0.3$), whereas the stratified monitor detects at the same budgets; the localization is what makes
the signal visible.

\subsection{Acquisition modes}
The budget can be spent two ways. Under a \emph{fixed budget} the strong judge is queried with
probability $\pi_i = \tau$ (allocated uniformly, by stratum, or by a bandit). Under
\emph{evidence-triggered escalation} the probability rises from a floor $\pi_0$ toward $1$ as a
stratum's accumulated e-wealth approaches its alarm threshold,
\begin{equation}
  \pi_i \;=\; \pi_0 + (1 - \pi_0)\cdot\operatorname{clip}\!\Big(\frac{\wealth_{i-1}}{\text{threshold}},\, 0,\, 1\Big).
  \label{eq:escalation}
\end{equation}
Crucially, $\pi_i$ is a function of \emph{past} wealth only, so it is predictable and anytime-validity
is preserved by construction. The trigger keys on accumulated e-evidence rather than on the cheap
covariate, because the cheap judge is blind to exactly the drift we are trying to catch.

\subsection{Cost model}
Each judge call costs $\text{(input tokens)}\cdot\text{(input price)} + \text{(output tokens)}\cdot
\text{(output price)}$, with representative counts of $\approx 600$ input and $\approx 30$ output
tokens. At the measured prices this is $\$0.000195$ per cheap call and $\$0.00156$ per strong call
(the cheap judge runs on all $L$ processed items; the strong judge on $S$ sampled items).
Relative to strong-evaluating every item, the cost fraction is
\begin{equation}
  \text{cost-fraction} \;=\; \frac{L\, c_{\cheap} + S\, c_{\strong}}{L\, c_{\strong}}
  \;=\; 0.125 + \frac{S}{L},
  \label{eq:cost}
\end{equation}
where full evaluation ($\pi \equiv 1$) is $1.0$ by definition. The $0.125$ is the irreducible cheap
floor and $S/L$ is the realized strong-sampling rate.

\subsection{What the foundation buys}
On a semi-synthetic localized drift, every variant of the detector catches the regression at
cost-fraction $0.21$--$0.43$ --- about one third of full evaluation --- with measured false-alarm
$0.000$ across all $16$ policy points. The two modes trace a cost--reliability frontier: escalation
occupies the cheap corner ($0.21$--$0.28$, sitting near the $0.125 + \pi_0$ floor when nothing is
wrong), fixed budgets buy reliability (fixed@$0.3$ reaches censoring $0.28$ at cost $0.43$), and full
evaluation sets the latency floor (median $327$ steps to detection) at full price. On the
cost--detection plane (\cref{fig:cost}) this ordering is visible on both datasets: evidence-triggered
escalation lies up-and-left of the fixed-budget family --- more detection per strong-judge dollar once
$\pi_0$ clears the hard-drift floor --- while paying in detection \emph{latency}, its weak axis, which
the plot omits and \cref{tab:frontier} records. We defer the
frontier table and the full ten-method comparison to \cref{app:c1}. One contrast from that comparison
matters here: keying acquisition on the cheap covariate is wasteful or inert --- a boundary-proximity
trigger detects only by oversampling to cost $\approx 0.87$--$0.89$. The published
predicted-e-growth rule, meanwhile, reads identically zero under an overrating cheap judge, whereas
e-wealth escalation detects at roughly one third of that cost. On faint drift, escalation merely ties a matched
fixed budget, consistent with the result of \citet{sfyraki2026adaptive} that uncertainty-driven
sampling is asymptotically no better than fixed-$\tau$ for prediction-powered mean estimation.

\begin{figure}[t]
\centering
\includegraphics[width=\textwidth]{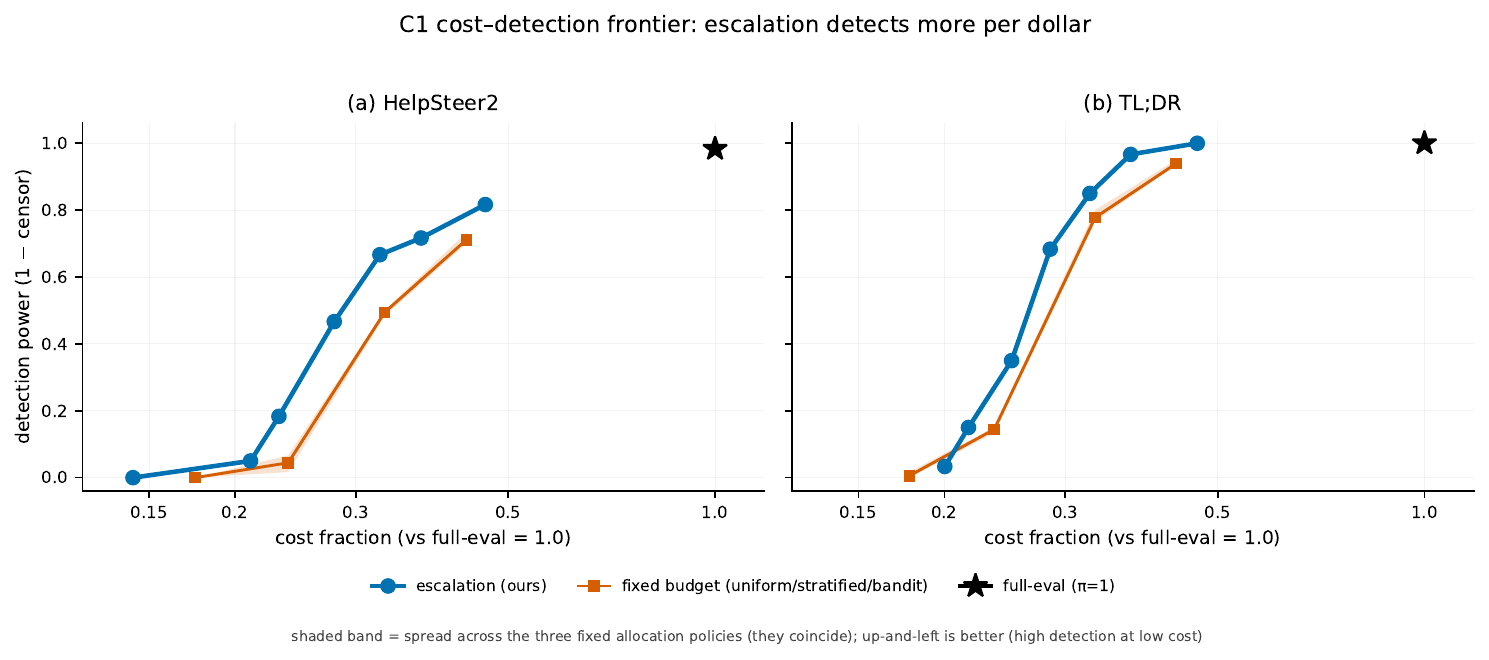}
\caption{The C1 cost--detection frontier on HelpSteer2 (left) and TL;DR (right): detection power
($1 - \text{censor}$) against cost-fraction (log axis; full strong-evaluation $= 1.0$).
Evidence-triggered escalation (ours, blue) lies up-and-left of the fixed-budget family --- whose
three allocation policies (uniform, stratified, bandit) coincide, so they are drawn as one centroid
line with a hairline min--max band --- detecting more per strong-judge dollar across the $\pi_0$
ladder; full evaluation ($\star$) buys the final increment of power at full price. Detection latency,
escalation's weak axis, is omitted here and tabulated per policy in \cref{tab:frontier}. The TL;DR
panel is the nothing-re-tuned replication detailed in \cref{sec:experiments}.}
\label{fig:cost}
\end{figure}

Everything above rests on trusting $\strong$. The next section drops that assumption.

\section{Disentangling Judge Drift: the Anchor Construction}\label{sec:anchor}


\subsection{The problem with a drifting reference}
The monitor of \cref{sec:setting} treats the strong judge $\strong$ as ground truth. But $\strong$ is
itself an LLM service: a version bump or a silent prompt/policy update can change its scoring without
notice. When that happens, a main-process alarm is ambiguous --- did the system get worse, or did the
judge change? --- and the failure is compounding: the contaminated $\strong$ feeds the very process
meant to certify system drift, so the detector can confidently report a system regression that is
really an artifact of its own reference moving. We disentangle the two with a second, independent
e-process anchored to fixed human labels.

\subsection{The anchor set}
\begin{definition}
An \emph{anchor set} $\anchorset = \{a_1, \dots, a_k\}$ is a fixed, held-out collection of items with
frozen human labels $\human(a) \in [0,1]^R$, drawn before monitoring starts and excluded from all
main-process calibration. At every position $i \equiv 0 \pmod{\rho}$ (an \emph{interleave rate}
$1/\rho$), one anchor $a$ is drawn uniformly from $\anchorset$ and re-scored by the \emph{current}
strong judge, yielding $\strong_i(a)$.
\end{definition}
Because the anchor items and their human labels never change, the only thing that can move
$\strong_i(a)$ is the judge. We summarize each re-judgement by a rescaled gap statistic
\begin{equation}
  \gapstat_i \;=\; \operatorname{clip}\!\Big(\tfrac{\strong_i(a) - \human(a) + 1}{2},\, 0,\, 1\Big) \;\in\; [0,1]^R,
  \label{eq:gap}
\end{equation}
with baseline $\gapbase \in (0,1)^R$ (a per-rubric anchor baseline, not to be confused with the
per-cell quality bars $q_0[r,t]$ of \cref{sec:setting}), the mean rescaled gap of the \emph{time-0}
judge over the full anchor set. The baseline is computed once before monitoring; because the time-0
judge scores are cached, $\gapbase$ is the exact mean over $\anchorset$ and carries no finite-sample
error.
\emph{This is the asymmetry the construction exploits: system drift changes the incoming items but
leaves the fixed anchors untouched, whereas a judge change moves both the main scores and the anchors
--- the anchors see judge drift and only judge drift.}

\subsection{The anchor e-process}
For each rubric $r$ and direction $d \in \{\text{below}, \text{above}\}$ we run a betting e-process
\citep{waudbysmith2024betting, csillag2025ppe} with e-values
\begin{equation}
  e_i \;=\; 1 + \lambda_i\, s_d\,\big(\gapbase[r] - \gapstat_i[r]\big),
  \qquad s_{\text{below}} = +1,\; s_{\text{above}} = -1,
  \label{eq:anchor-evalue}
\end{equation}
where the bet $\lambda_i \in [0, \lambda_{\max}(\gapbase[r], d)]$ is \emph{predictable} (a function of
$\gapstat_1, \dots, \gapstat_{i-1}$ only; in code the plug-in $\lambda \approx
\operatorname{mean}(d)/\operatorname{mean}(d^2)$). The wealth is $\wealth_n^{(r,d)} = \prod_{i \leq n} e_i$,
and the family alarms when $\max_{r,d} \wealth_n^{(r,d)} \geq 2R/\alphajudge$ (Bonferroni over the $2R$
rubric--direction pairs). The two directions test, respectively, that the judge has become more lenient
or more harsh on the anchors relative to the human labels.

\subsection{The attribution rule}
Let $\tau_{\mathrm{sys}}$ and $\tau_{\mathrm{anc}}$ be the first alarm positions of the main and anchor
families ($\infty$ if neither ever fires), and fix a guard window $\guard \geq 0$. The verdict is
\begin{equation}
  \text{verdict} =
  \begin{cases}
    \texttt{none}   & \text{if neither family fires,}\\[2pt]
    \texttt{judge}  & \text{if } \tau_{\mathrm{anc}} < \infty \text{ and } (\tau_{\mathrm{sys}} = \infty \text{ or } \tau_{\mathrm{anc}} \leq \tau_{\mathrm{sys}} + \guard),\\[2pt]
    \texttt{system} & \text{if } \tau_{\mathrm{sys}} < \infty \text{ and } (\tau_{\mathrm{anc}} = \infty \text{ or } \tau_{\mathrm{anc}} > \tau_{\mathrm{sys}} + \guard).
  \end{cases}
  \label{eq:attribution}
\end{equation}
The rule tests two nulls. The anchor null $H_0^{\mathrm{anc}}$ (no judge change) states that,
conditional on the past, each observed rescaled gap has mean $\gapbase[r]$ for every rubric --- the
current judge's gap-to-human distribution on the fixed anchor set matches the time-0 baseline. The
system null $H_0^{\mathrm{sys}}$ is the per-stratum main-process null of \cref{eq:sys-null}. When both
families fire within the guard window the verdict is \texttt{judge} \emph{by design}: a contaminated
$\strong$ invalidates the main alarm, so we do not attempt simultaneous attribution.

\subsection{Guarantees}
The validity claims below are inherited from standard test-supermartingale results; the contribution is
the construction these results apply to, not new probability theory. We state each proposition and a
one-line proof idea, deferring full proofs to \cref{app:proofs}.

\begin{proposition}[Anytime-validity of the anchor process]\label{prop:validity}
Under $H_0^{\mathrm{anc}}$, for every rubric--direction pair $(\wealth_n^{(r,d)})_n$ is a nonnegative
supermartingale with $\wealth_0 = 1$, and consequently
\begin{equation*}
  \mathbb{P}\Big(\exists\, n :\ \max_{r,d} \wealth_n^{(r,d)} \geq 2R/\alphajudge\Big) \;\leq\; \alphajudge,
\end{equation*}
uniformly over all stopping times: at any data-dependent monitoring horizon, the probability that the
anchor family \emph{ever} falsely fires is at most $\alphajudge$.
\end{proposition}
\begin{proof}[Proof idea]
Boundedness of $\gapstat_i$ and the directional bound on $\lambda_i$ make each $e_i \geq 0$;
predictability of $\lambda_i$ together with $\mathbb{E}[\gapstat_i[r] \mid \text{past}] = \gapbase[r]$
under $H_0^{\mathrm{anc}}$ gives $\mathbb{E}[e_i \mid \text{past}] = 1$, so each wealth is a nonnegative
supermartingale. Ville's inequality \citep{ville1939} bounds each pair's crossing probability by
$\alphajudge/2R$, and a union bound over the $2R$ pairs finishes (full proof in \cref{app:proofs}).
\end{proof}
\begin{remark}
The bound needs no distributional assumptions beyond boundedness --- the gaps need not be Gaussian,
i.i.d.\ across rubrics, or stationary in variance --- and no held-out calibration streams or fixed
horizon, in contrast to classical change detectors that must set a threshold against $H_0$ streams and
are calibrated only at that horizon.
\end{remark}

\begin{proposition}[Identification]\label{prop:identification}
The anchor observations depend only on the fixed items $\anchorset$, their frozen labels $\human$, and
the current judge. System drift --- any change in the distribution of new items or their true quality
--- leaves the anchor gap distribution unchanged, so $H_0^{\mathrm{anc}}$ can be violated only by a
change in the judge. Consequently: \textup{(a)} under pure system drift,
$\mathbb{P}(\text{verdict} = \texttt{judge}) \leq \alphajudge$, uniformly in $\guard$; and
\textup{(b)} under no drift at all, $\mathbb{P}(\text{verdict} \neq \texttt{none}) \leq \alphasys +
\alphajudge$.
\end{proposition}
\begin{proof}[Proof idea]
A false \texttt{judge} verdict requires the anchor family to fire, whose probability \cref{prop:validity}
bounds by $\alphajudge$ irrespective of $\guard$; (b) is a union bound over the two families' separate
false-fire events. See \cref{app:proofs}.
\end{proof}
\begin{remark}
This is the disentanglement asymmetry: the anchor process is a \emph{one-way test} for judge identity.
Its converse direction --- catching judge drift quickly enough to win the race against a contaminated
main process --- is power, not validity, and is the subject of \cref{prop:race}.
\end{remark}

\begin{proposition}[The attribution race]\label{prop:race}
Under judge drift (with or without simultaneous system drift) the misattribution event
$\{\text{verdict} = \texttt{system}\}$ requires the contaminated main process to fire more than $\guard$
positions before the anchors:
\begin{equation*}
  \mathbb{P}(\text{verdict} = \texttt{system} \mid \text{judge drift}) \;=\; \mathbb{P}(\tau_{\mathrm{sys}} + \guard < \tau_{\mathrm{anc}}),
\end{equation*}
which is monotonically non-increasing in the guard width $\guard$, in the anchor interleave rate
$1/\rho$, and in the anchor set size $k$ (more anchor evidence per position yields stochastically
smaller $\tau_{\mathrm{anc}}$), and monotonically non-decreasing in the main process's sampling power
(more strong calls on a contaminated $\strong$ yield stochastically smaller $\tau_{\mathrm{sys}}$).
\end{proposition}
\begin{proof}[Proof idea]
The event $\{\text{verdict} = \texttt{system}\}$ under judge drift is exactly
$\{\tau_{\mathrm{sys}} + \guard < \tau_{\mathrm{anc}}\}$ by \cref{eq:attribution}; the monotonicities
follow from stochastic ordering of the stopping times in each knob. See \cref{app:proofs}.
\end{proof}
\begin{remark}
The guard window costs attribution latency but never validity: raising $\guard$ delays \texttt{system}
attribution yet cannot manufacture false \texttt{judge} verdicts beyond the $\alphajudge$ of
\cref{prop:identification}(a), precisely because the anchor bound holds at \emph{any} horizon. The race
also leaves detection untouched: the total event $\{\text{drift flagged at all}\} =
\{\tau_{\mathrm{sys}} \wedge \tau_{\mathrm{anc}} < \infty\}$ does not depend on $\guard$; the rule only
decides the label.
\end{remark}

\begin{proposition}[Orthogonality]\label{prop:orthogonality}
$\tau_{\mathrm{anc}}$ is independent of the main-process configuration (policy, budget, $\pi$), and the
main process's distribution is independent of the anchor interleave: anchor re-judgements are
out-of-band calls on held-out items that never enter the main e-processes or their calibration. The
design therefore splits cleanly --- the anchor budget $(k, \rho)$ and guard $\guard$ govern judge-drift
latency and the race, the main configuration governs system power, and $\alphajudge, \alphasys$ are
separate Bonferroni budgets.
\end{proposition}
\begin{proof}[Proof idea]
The anchor observations are a deterministic function of $(\anchorset, \human, \text{current judge})$
and the interleave schedule, none of which depend on the main acquisition rule; symmetrically, held-out
anchor calls never enter the main statistics. See \cref{app:proofs}.
\end{proof}
\begin{remark}
Orthogonality holds for the two \emph{processes}; the \emph{verdict} still couples them through the race
of \cref{prop:race}.
\end{remark}

We verify these guarantees empirically in \cref{sec:experiments}: \cref{prop:validity} and
\cref{prop:identification} as false-\texttt{judge} rates at or below $\alphajudge$ under pure system
drift and no drift; \cref{prop:race} through the guard-window sweep and the cost frontier; and
\cref{prop:orthogonality} through the constant judge-latency rows and constant system-power columns of
that frontier. Two boundaries are worth stating plainly. The construction claims no optimality --- the
plug-in $\lambda$ approximates the log-optimal bet, and the attribution rule is not claimed to minimize
misattribution at a fixed budget --- and it assumes the human labels $\human$ remain a valid reference:
if the anchor items' true quality standard itself drifts conceptually, the anchor process correctly
reports judge-vs-anchor disagreement, but its reading as judge drift weakens until a periodic anchor
refresh re-establishes the baseline.

\section{Experiments}\label{sec:experiments}


We evaluate four claims empirically: the attribution rule disentangles judge drift from
system drift (\cref{prop:identification}); the misattribution rate is a tunable race in the
guard width, anchor rate, and anchor size (\cref{prop:race}); the two processes are
orthogonal and must be co-provisioned (\cref{prop:orthogonality}); and the anytime-valid
anchor process dominates the classical alternatives operationally. Every experiment is run
on two datasets --- HelpSteer2 assistant-response evals and TL;DR summarization evals ---
with \emph{nothing} re-tuned between them: the same judge roster, featured configuration,
thresholds, and $\alpha$ budgets. The second dataset thereby functions as an out-of-sample
test of the propositions, not only a robustness check: invariances must transfer, and
cross-domain differences must land where the theory says (\cref{sec:exp-generalize}). We
close by showing every property also generalizes across all topic--rubric cells of both
datasets.

\subsection{Setup}\label{sec:exp-setup}

We use two datasets. The primary is HelpSteer2 \citep{helpsteer2}: $1498$ items, each scored
on $R = 5$ rubrics (helpfulness, correctness, coherence, complexity, verbosity) and carrying
a topic label over $K = 6$ strata. The second is the TL;DR summarization axis evals
\citep{stiennon2020tldr}: $1498$ summarization items scored on $R = 4$ axes (overall,
accuracy, coverage, coherence) with subreddit strata ($K = 5$). On both, the cheap judge
$\cheap$ is \texttt{gemini-3.1-flash-lite} and the strong judge $\strong$ is
\texttt{gemini-3.1-pro-preview} \citep{gemini}; the frozen anchor labels $\human$ are each
dataset's human ratings (HelpSteer2 per-attribute $0$--$4$, TL;DR $7$-point Likert, both
rescaled to $[0,1]$ --- the anchor baseline gap absorbs the constant scale offset). We
monitor streams of length $N = 1200$ with the change point at position $300$, averaging over
$60$ seeds, and set $\alphasys = \alphajudge = 0.1$.

We inject three judge drifts spanning the realism axis. \emph{Real-lenient}: every item is
re-judged by \texttt{gemini-3.5-flash}, an actual silent version bump, which runs uniformly
more lenient than the baseline judge (per-rubric mean gap shift $+0.022$ to $+0.072$ on
HelpSteer2, $+0.031$ to $+0.071$ on TL;DR --- the same bump, the same signature, in both
domains). \emph{Real-harsh}: the \emph{same} model is re-prompted with a stricter scoring
prompt (``v2-strict''), a real judge-behavior change that shifts scores down by $-0.043$ to
$-0.152$ on HelpSteer2 (verbosity the lone exception at $+0.025$) and by $-0.097$ to $-0.230$
on TL;DR (summaries offer more flaws to penalize). \emph{Synthetic}: a tunable shift $-0.25$
plus noise. System drift is injected by drawing the target cell (creative$\times$helpfulness
on HelpSteer2, relationships$\times$overall on TL;DR) post-change from a regressed-blind pool
($\strong$ low, $\cheap$ high), the exact blind spot that makes the cheap judge alone
insufficient. Crossing the two drifts gives four ground-truth conditions
$\{\texttt{none}, \texttt{system}, \texttt{judge}, \texttt{both}\}$.

Unless a knob is swept, we report a single \emph{featured configuration}: a fixed-budget main
process at sampling rate $0.3$ (the high-power point on the cost frontier of \cref{sec:setting})
with $k = 200$ anchors interleaved at rate $1/5$ and guard $\guard = 300$ (justified in
\cref{sec:exp-coprovision}). This configuration runs at cost-fraction $\approx 0.64$
(HelpSteer2) / $0.63$--$0.65$ (TL;DR) of strong-evaluating every item (the $0.125$ cheap
floor, $0.30$ main sampling, and $1/5$ anchor share), versus $1.0$ for full evaluation; a
cheap-but-deaf alternative (escalation $\pi_0 = 0.02$ with rate-$20$ anchors) costs $0.21$
but misses most system regressions.

\subsection{Disentanglement and detection}\label{sec:exp-disentangle}

\cref{tab:confusion} reports the verdict confusion matrices under the synthetic harsh shift
on both datasets. With matched anchors, system detection is $41/60$ (HelpSteer2) and $56/60$
(TL;DR) --- exactly each fixed@$0.3$ main's power from \cref{sec:setting} --- while judge and
both are detected $60/60$ and $55$--$59/60$. On HelpSteer2 the pure \texttt{judge} and
\texttt{both} rows \emph{never} land in \texttt{system} (judge$\to$system $0/120$); on TL;DR
a small spill remains at the default guard ($5{+}1/120$) --- the guard race of
\cref{sec:exp-harsh}, where widening $\guard$ closes it. False \texttt{judge} verdicts stay
at or below $\alphajudge$ on both datasets: $2/60$ on \texttt{none} and $2/60$ on
\texttt{system} (rate $\approx 0.03 < 0.1$, as expected --- the rate-$5$ interleave gives the
anchor family four times the observations, so a few $\alpha$-budget false fires appear by
design).

The right half of \cref{tab:confusion} is the co-provisioning lesson. With
under-provisioned anchors ($k = 50$, rate $1/20$), the high-power main reads the contaminated
$\strong$ so fast that it fires roughly $300$ items before the anchors catch up:
$48$--$50/60$ (HelpSteer2) and $54$--$58/60$ (TL;DR) of the genuine judge-drift runs land in
\texttt{system} and the disentanglement collapses. The anchor process must be provisioned to
\emph{out-run the main process it guards} --- anchor detection latency ($120$/$140$ items for
the matched anchors, versus $420$/$440$ for the slow ones) has to come in under the main's
contaminated-fire latency plus $\guard$.

\begin{table}[t]
\centering\small
\caption{Verdict confusion (60 seeds, synthetic judge shift $-0.25$, fixed@0.3
main, guard $\guard{=}50$), both datasets. Left: matched anchors ($k{=}200$, rate
$1/5$). Right: under-provisioned anchors ($k{=}50$, rate $1/20$) --- the
contaminated main out-runs the anchors and the disentanglement collapses on both
datasets.}
\label{tab:confusion}
\begin{tabular}{l rrr c rrr}
\toprule
& \multicolumn{3}{c}{$k{=}200$, $1/5$ (matched)} & &
  \multicolumn{3}{c}{$k{=}50$, $1/20$} \\
\cmidrule{2-4}\cmidrule{6-8}
true $\backslash$ pred & none & system & judge & & none & system & judge \\
\midrule
\multicolumn{8}{l}{\emph{HelpSteer2} (median judge latency $120$ / $420$)} \\
none   & 58 & 0  & 2  & & 60 & 0  & 0  \\
system & 17 & 41 & 2  & & 20 & 40 & 0  \\
judge  & 0  & 0  & 60 & & 0  & 48 & 12 \\
both   & 0  & 0  & 60 & & 0  & 50 & 10 \\
\midrule
\multicolumn{8}{l}{\emph{TL;DR} (median judge latency $140$ / $440$)} \\
none   & 58 & 0  & 2  & & 60 & 0  & 0  \\
system & 2  & 56 & 2  & & 2  & 58 & 0  \\
judge  & 0  & 5  & 55 & & 0  & 58 & 2  \\
both   & 0  & 1  & 59 & & 0  & 54 & 6  \\
\bottomrule
\end{tabular}
\end{table}

\cref{tab:realbump} turns to a \emph{real} silent version bump and the anchor-budget
frontier. The lenient \texttt{gemini-3.5-flash} bump is moderate (mean shift $\leq +0.072$,
magnitude comparable to the synthetic $-0.05$/$-0.10$ sweep points), so detection depends on
anchor budget --- on both datasets: $7\%$ at $k = 50$/rate-$20$ on each, rising to $100\%$
($60/60$) on HelpSteer2 and $83\%$ ($50/60$) on TL;DR at the featured $k = 200$/rate-$5$
(the bump is subtler per-axis on TL;DR's featured cell, so the same budget buys less of the
frontier; misses fail safe to \texttt{none}). The pure-\texttt{judge} row never lands in
\texttt{system} at any budget on either dataset, and false-judge rates stay under
$\alphajudge$. The featured fixed@$0.3$ replications hold: judge detection $60/60$ at latency
$498$ (HelpSteer2) and $50/60$ at latency $555$ (TL;DR), judge$\to$system $0$.

\begin{table}[t]
\centering\small
\caption{A real version bump (gemini-3.1-pro $\to$ 3.5-flash) vs anchor budget, both
datasets. 60 seeds; escalation main (featured fixed@0.3 replications: 60/60 at latency
498 on HelpSteer2, 50/60 at latency 555 on TL;DR, judge$\to$system 0 on both).}
\label{tab:realbump}
\begin{tabular}{llcccc}
\toprule
anchors & rate & detection (judge / both) & judge$\to$sys & none$\to$judge & latency \\
\midrule
\multicolumn{6}{l}{\emph{HelpSteer2} (mean shift $\leq{+}0.072$)} \\
$k{=}50$  & $1/20$ & 7\% (4/60) / 7\%        & 0 & 0/60 & 660 \\
$k{=}150$ & $1/8$  & 62\% (37/60) / 60\%     & 0 & 1/60 & 652 \\
$k{=}200$ & $1/5$  & \textbf{100\% (60/60)} / 98\% & 0 & 2/60 & 498 \\
\midrule
\multicolumn{6}{l}{\emph{TL;DR} (mean shift $\leq{+}0.071$)} \\
$k{=}50$  & $1/20$ & 7\% (4/60) / 7\%        & 0 & 0/60 & 810 \\
$k{=}150$ & $1/8$  & 38\% (23/60) / 38\%     & 0 & 2/60 & 516 \\
$k{=}200$ & $1/5$  & \textbf{83\% (50/60)} / 83\% & 0 & 2/60 & 555 \\
\bottomrule
\end{tabular}
\end{table}

\paragraph{The one-sided caveat.} The lenient bump scores $\strong' \geq \strong$ on every
rubric, and the main monitor is one-sided ``below'' --- it fires only on quality
\emph{drops}. A lenient judge therefore cannot push the main toward a \texttt{system} alarm at
all, so the judge$\to$system $= 0$ of \cref{tab:realbump} is \emph{partly structural} for this
drift direction, not solely the anchors' doing. \cref{tab:realbump} thus demonstrates
\emph{detection} of a subtle real bump; the genuine test of disentanglement-under-contamination
is a \emph{harsher} judge change that does push the main downward, which is the subject of
\cref{sec:exp-harsh}. Both demonstrations are needed; neither alone supports the full claim.

\subsection{A real harsh judge change: contamination and the guard race}\label{sec:exp-harsh}

The v2-strict re-judge is a real LLM judge-behavior change: the same model is instructed to
``reserve $4$ for flawless work \dots\ when torn, choose the lower.'' \cref{tab:judgegap}
characterizes both real drifts per rubric on both datasets --- the lenient bump is a uniform
leniency shift ($\leq +0.072$) with the \emph{same} $+0.03$ to $+0.07$ signature in both
domains (a property of the version bump, not the dataset), while the harsh re-judge shifts
strictly downward: into the $-0.10$ to $-0.15$ band on HelpSteer2 (verbosity the lone,
near-noise exception at $+0.025$) and deeper, $-0.10$ to $-0.23$, on TL;DR. This is the
synthetic sweep's hardest regime: slow anchors against a fast contaminated main. Note that
item-level agreement degrades far more than the means suggest (exact agreement
$26$--$84\%$), and the rubrics where the two strong judges agree least (verbosity,
complexity, corr $0.46$--$0.60$) mirror the cheap judge's own blind spots --- the anchor
process fires on the systematic component, not the per-item noise, which is why it needs
$k \approx 200$ anchors for the subtle bump.

\begin{table}[t]
\centering\small
\caption{Strong-judge gap $\strong' - \strong$ per rubric, both real drifts, both datasets
($\approx 1496$--$1498$ items judged by both judges). \emph{shift} is the mean gap shift the
anchor process monitors; \emph{corr} and \emph{agree\%} are item-level. The lenient bump
moves scores up uniformly with the same signature in both domains; the harsh re-judge moves
them down --- harder on TL;DR, where summaries offer more flaws to penalize.}
\label{tab:judgegap}
\begin{tabular}{l ccc c ccc}
\toprule
& \multicolumn{3}{c}{lenient (3.5-flash)} & & \multicolumn{3}{c}{harsh (v2-strict)} \\
\cmidrule{2-4}\cmidrule{6-8}
rubric & shift & corr & agree\% & & shift & corr & agree\% \\
\midrule
\multicolumn{8}{l}{\emph{HelpSteer2}} \\
helpfulness & $+0.064$ & 0.88 & 63.6\% & & $-0.152$ & 0.90 & 45.0\% \\
correctness & $-0.008$ & 0.86 & 70.0\% & & $-0.120$ & 0.88 & 55.1\% \\
coherence   & $+0.022$ & 0.76 & 82.5\% & & $-0.092$ & 0.79 & 62.9\% \\
complexity  & $+0.050$ & 0.59 & 72.5\% & & $-0.043$ & 0.57 & 77.5\% \\
verbosity   & $+0.072$ & 0.46 & 62.6\% & & $+0.025$ & 0.60 & 71.0\% \\
\midrule
\multicolumn{8}{l}{\emph{TL;DR}} \\
overall     & $+0.048$ & 0.89 & 69.7\% & & $-0.230$ & 0.84 & 26.2\% \\
accuracy    & $+0.031$ & 0.91 & 84.4\% & & $-0.097$ & 0.89 & 65.7\% \\
coverage    & $+0.071$ & 0.87 & 65.6\% & & $-0.193$ & 0.83 & 31.9\% \\
coherence   & $+0.037$ & 0.81 & 78.6\% & & $-0.134$ & 0.78 & 51.0\% \\
\bottomrule
\end{tabular}
\end{table}

\cref{tab:guard} reports the featured config under this real contamination, sweeping the guard
window $\guard$ on both datasets. The harsh shift genuinely contaminates the one-sided main,
but the anchors flip the large majority of contaminated fires to the correct \texttt{judge}
verdict, and the residual spill is a tunable race, not a defect: judge$\to$system shrinks
monotonically in $\guard$ --- $29 \to 18 \to 10$ of $120$ on HelpSteer2 and $5 \to 1 \to 0$ on
TL;DR --- exactly the monotonicity of \cref{prop:race}, because spills happen only when the
contaminated main fires more than $\guard$ items before the anchors. TL;DR reaches a
\emph{perfect} $240/240$ over its two drift rows at $\guard = 300$: its harsh shift is larger
(overall $-0.23$), so the anchors are faster (median latency $205$ versus $370$) and win the
race outright --- a bigger judge change is \emph{easier} to attribute. Critically, the
\texttt{none} and \texttt{system} rows are identical at every $\guard$ on both datasets
($58/0/2$ and $17/41/2$; $58/0/2$ and $2/56/2$): widening the guard costs nothing in system
detection or false-judge rate --- only attribution latency and a sliver of anchor cost. The
guard window is the third knob, after anchor size and interleave rate, and the cheapest: it
buys correctness with latency rather than dollars.

\begin{table}[t]
\centering\small
\caption{Real harsh judge change (v2-strict re-judge), featured config (fixed@0.3 main,
$k{=}200$/rate-$5$ anchors), 60 seeds, sweeping guard $\guard$, both datasets. Spill =
judge$\to$system total over the two drift rows (of 120). The \texttt{none} and \texttt{system}
rows are identical at every $\guard$ on both datasets: guard width costs only attribution
latency. TL;DR's larger shift makes the anchors faster, closing the spill entirely at
$\guard{=}300$.}
\label{tab:guard}
\begin{tabular}{c ccc cc}
\toprule
guard $\guard$ & judge row (none/sys/judge) & both row (none/sys/judge) & spill /120 & latency & cost \\
\midrule
\multicolumn{6}{l}{\emph{HelpSteer2} (shifts $-0.04$ to $-0.15$)} \\
50  & 0 / 10 / 50 & 0 / 19 / 41 & 29 & 370 & 0.64 \\
150 & 0 / 8 / 52  & 0 / 10 / 50 & 18 & 370 & 0.66--0.67 \\
300 & 0 / \textbf{3} / \textbf{57} & 0 / \textbf{7} / \textbf{53} & \textbf{10} & 370 & 0.68--0.71 \\
\midrule
\multicolumn{6}{l}{\emph{TL;DR} (shifts $-0.10$ to $-0.23$)} \\
50  & 0 / 3 / 57 & 0 / 2 / 58 & 5 & 205 & 0.63--0.65 \\
150 & 0 / 1 / 59 & 0 / 0 / 60 & 1 & 205 & 0.63--0.68 \\
300 & 0 / \textbf{0} / \textbf{60} & 0 / \textbf{0} / \textbf{60} & \textbf{0} & 205 & 0.63--0.73 \\
\bottomrule
\end{tabular}
\end{table}

\cref{fig:race} shows one seed of the featured config under this real harsh drift on each
dataset. The change point is at $300$. On HelpSteer2 the anchor process crosses its threshold
($\approx 4.61$) at position $735$, well before the contaminated main crosses its threshold
($\approx 5.70$) at $1098$; on TL;DR the larger shift accelerates \emph{both} processes ---
anchors at $370$, contaminated main at $567$ --- but the anchors keep winning. In both panels
the verdict is \texttt{judge}: the contaminated main fire lands inside the guard window
(shaded) opened by the anchor alarm.

\begin{figure}[t]
\centering
\includegraphics[width=\textwidth]{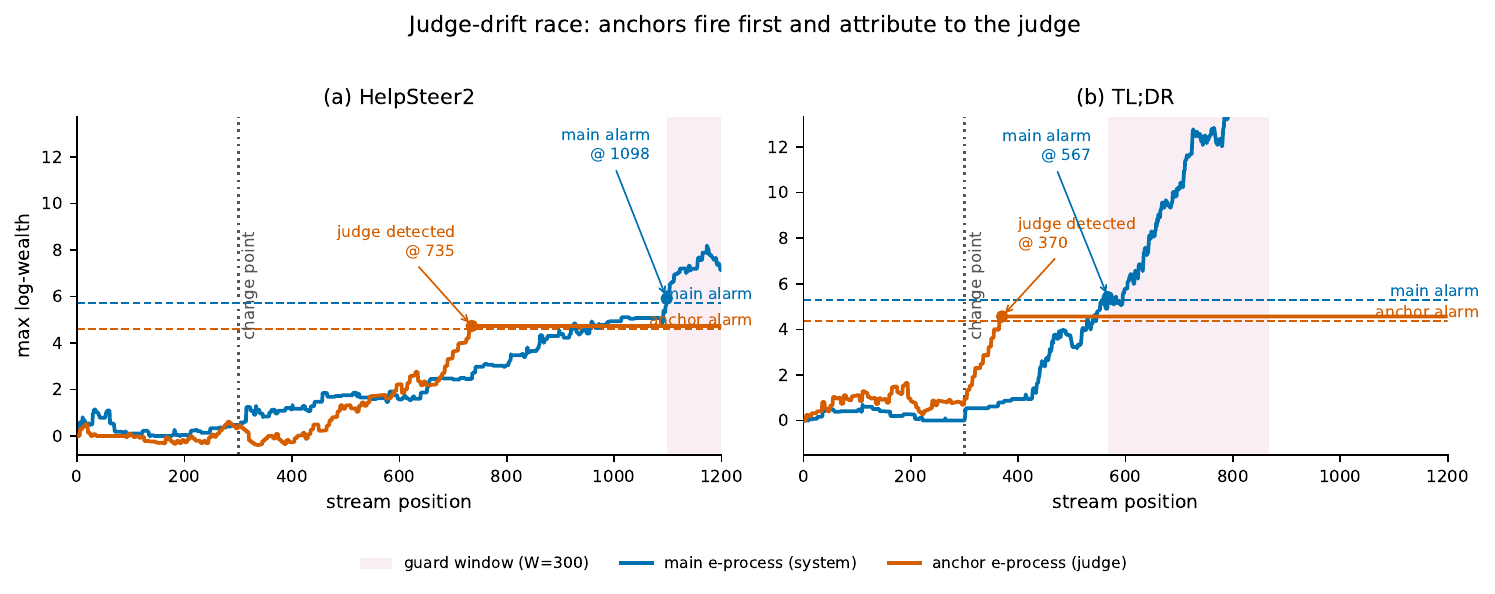}
\caption{The attribution race under real harsh judge drift (one seed, featured config), on
HelpSteer2 (left) and TL;DR (right). Both e-process wealth trajectories are shown against
their Bonferroni alarm thresholds; the change point is at $300$. The anchor process alarms
first on both datasets ($735$ vs $1098$ on HelpSteer2; $370$ vs $567$ on TL;DR), so the main
fire falls inside the guard window (shaded) and the verdict is correctly \texttt{judge}.}
\label{fig:race}
\end{figure}

\subsection{Co-provisioning: main power $\times$ anchor rate}\label{sec:exp-coprovision}

\cref{fig:frontier} sweeps five main configurations against three anchor interleave rates
(real harsh drift, $k = 200$, $\guard = 50$, $30$ seeds) on both datasets, establishing four
structural facts. First, the two processes are \emph{orthogonal} as designed
(\cref{prop:orthogonality}): judge latency depends only on the anchor rate ($830/695/430$
items at rate $20/10/5$ on HelpSteer2; $640/330/202$ on TL;DR --- across every main config),
system power depends only on the main config (constant down each main's rows), and the
false-judge rate is flat at $\leq 0.07 \leq \alphajudge$. Second, at under-anchored points
the sum of judge-detection and spill is $\approx 1$: those points do not \emph{miss} the
drift, they \emph{mislabel} it --- detection and attribution are separate. Third, the spill
grows monotonically with main power at fixed rate --- on HelpSteer2's rate-$5$ column,
$0.03 \to 0.07 \to 0.18 \to 0.33$ (escalation $\pi_0 = 0.02$, fixed@$0.1$, fixed@$0.2$,
fixed@$0.3$) --- so the feasible region is a diagonal band and the anchor budget must be
co-provisioned with main power, not bolted on. The TL;DR panel shows the same geometry
entered from the fast-anchor side: its larger harsh shift lets rate-$5$ anchors out-run even
the fixed@$0.3$ main (spill $\leq 0.05$), and the band sits at rates $10$--$20$ instead
($0.42$ at fixed@$0.3$/rate-$10$, $0.90$ at rate-$20$) --- the race of \cref{prop:race}, with
the drift magnitude setting where the band falls. Fourth, the guard $\guard$ is the
orthogonal escape hatch: the sweep of \cref{sec:exp-harsh} pulled the HelpSteer2 corner's
spill from $29$ to $10$ of $120$ at $\guard = 300$ (and TL;DR's from $5$ to $0$), so points
outside the band are recoverable with attribution latency instead of anchor dollars. Together
these justify the featured configuration: a high-power main matched with rate-$5$ anchors and
a wide guard.

\begin{figure}[t]
\centering
\includegraphics[width=\textwidth]{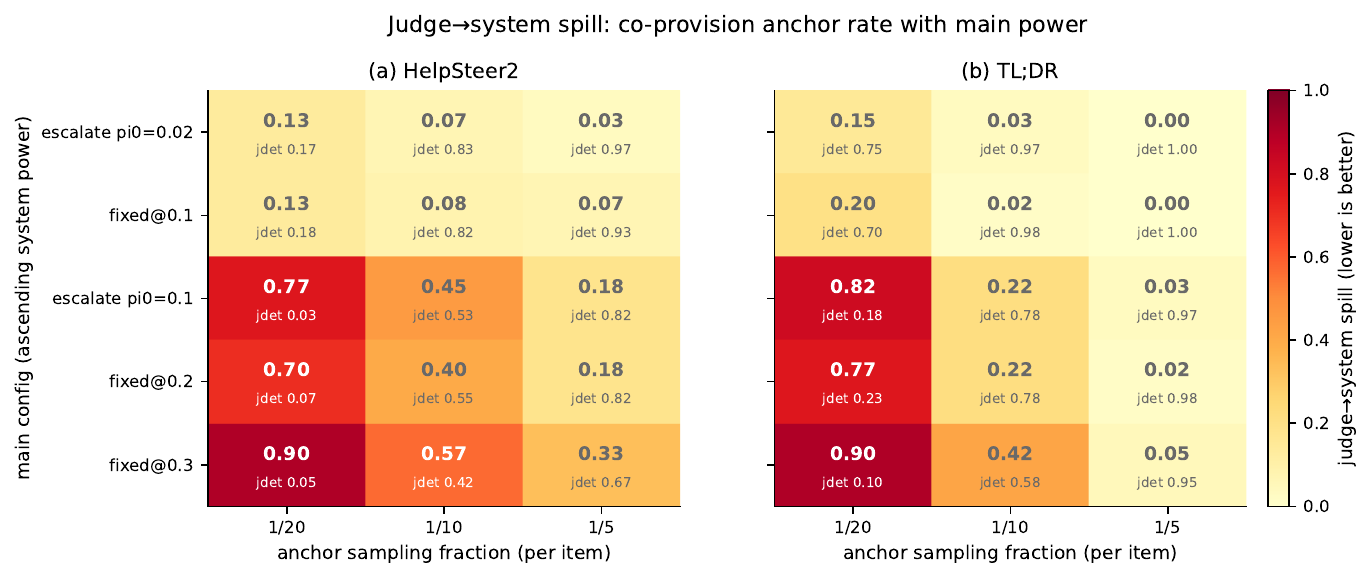}
\caption{Co-provisioning frontier (real harsh drift, $k{=}200$, $\guard{=}50$, 30 seeds), on
HelpSteer2 (left) and TL;DR (right). Judge latency depends only on the anchor rate and system
power only on the main config (\cref{prop:orthogonality}); the misattribution spill grows
along the main-power axis at fixed anchor rate, so the feasible region is a diagonal band ---
main power and anchor rate must be raised together. TL;DR's larger harsh shift shifts the
band toward slower anchor rates: the same geometry, set by the drift magnitude.}
\label{fig:frontier}
\end{figure}

\subsection{Baselines: why anytime-validity matters operationally}\label{sec:exp-baselines}

We compare the anchor e-process against two classical change detectors consuming the
\emph{identical} anchor observation stream ($k = 200$, rate $1/5$, $60$ seeds, both
datasets); \cref{tab:baselines} reports false-alarm rate, detection rate, and latency on each
drift. The default practice fails outright. \emph{naive-z} --- a rolling-window $z$-test
against baseline at $\alpha = 0.05$, re-tested every observation with no sequential
correction, the industry standard --- false-alarms on $75\%$ (HelpSteer2) and $67\%$ (TL;DR)
of no-drift streams, and $16/60$ and $10/60$ of its ``detections'' fire \emph{before the
change even happens}: its apparent speed is the absence of error control, not aggression.
\emph{ph-calib} is Page--Hinkley \citep{page1954} per rubric$\times$direction, with the
family threshold calibrated on $60$ held-out $H_0$ anchor streams to FWER $0.1$ --- the
strongest classical recipe. It holds its false-alarm budget ($0.08$/$0.03$) but is blind to
the subtle real lenient bump ($8\%$ on HelpSteer2, $40\%$ on TL;DR), and where it does detect
it is slower than the e-process ($295$ versus $120$ on the synthetic shift; $248$ versus
$205$ and $165$ versus $140$ on TL;DR). The anchor e-process dominates on both datasets:
false-alarm $0.03 \leq \alphajudge$ with no calibration data and a horizon-free guarantee,
$97\%$ detection on all three HelpSteer2 drifts and $83$--$100\%$ on TL;DR (the $83\%$ on
the lenient bump \emph{is} the anchor-budget frontier of \cref{tab:realbump}, not a method
limitation), and the best latency at every drift among false-alarm-controlled methods.
Anytime-validity is not a theoretical nicety: the default practice is unusable for continuous
monitoring, and the properly-calibrated classical alternative pays with both blindness to
subtle real drift and a calibration-data tax.

\begin{figure}[t]
\centering
\includegraphics[width=\textwidth]{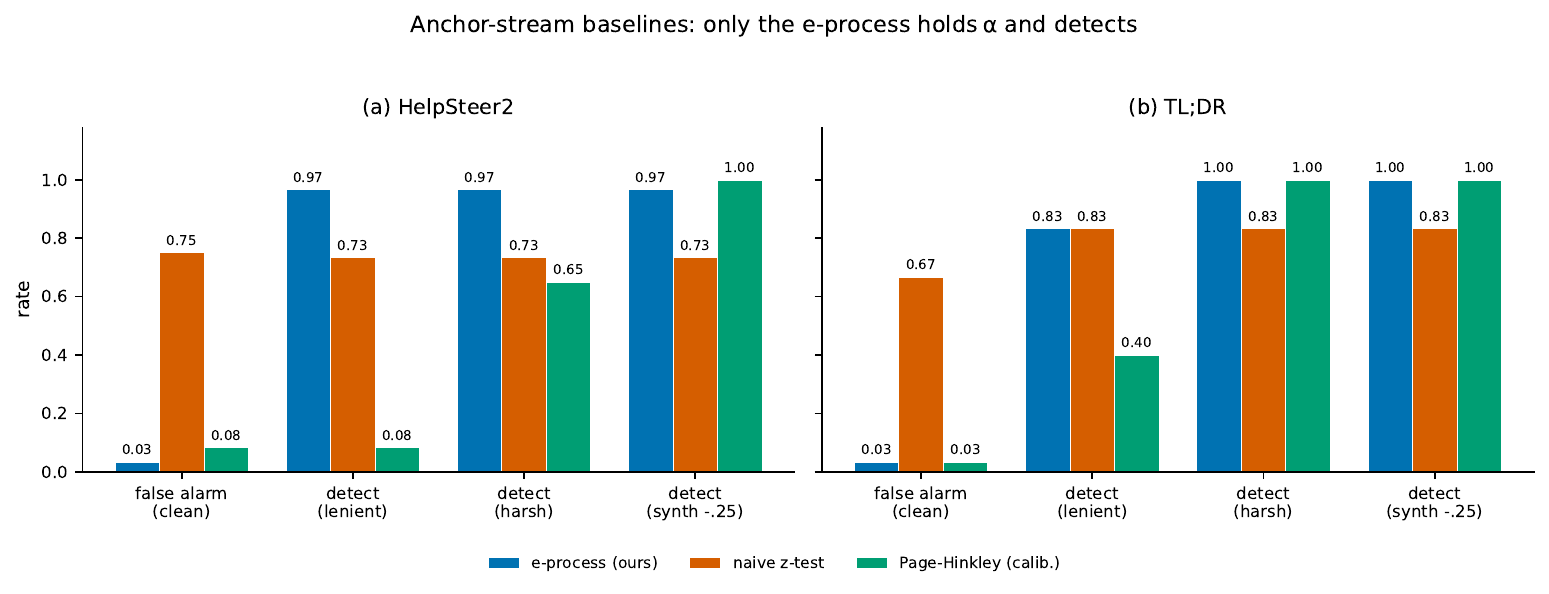}
\caption{The anchor e-process versus classical change detectors on the identical anchor
stream, on HelpSteer2 (left) and TL;DR (right). naive-z (rolling $z$-test, the industry
default) has no sequential error control and false-alarms on $75\%$/$67\%$ of no-drift
streams; calibrated Page--Hinkley holds $\alpha$ but misses the subtle real lenient bump; the
anytime-valid e-process controls false alarms at $\alphajudge$ with no calibration data and
detects all drifts up to its anchor budget.}
\label{fig:baselines}
\end{figure}

\begin{table}[t]
\centering\small
\caption{Anchor-process baselines on the identical anchor stream (60 seeds), both datasets.
FA = false-alarm rate on no-drift; detection rate / median latency on each of the three judge
drifts. naive-z additionally fires $16/60$ (HelpSteer2) and $10/60$ (TL;DR) alarms
\emph{before} the change point.}
\label{tab:baselines}
\begin{tabular}{l c ccc ccc}
\toprule
& & \multicolumn{3}{c}{detection rate} & \multicolumn{3}{c}{median latency} \\
\cmidrule{3-5}\cmidrule{6-8}
method & FA & lenient & harsh & synth & lenient & harsh & synth \\
\midrule
\multicolumn{8}{l}{\emph{HelpSteer2}} \\
eprocess  & 0.03 & 0.97 & 0.97 & 0.97 & 498 & 370 & 120 \\
naive-z   & 0.75 & 0.73 & 0.73 & 0.73 & 130 & 138 & 40  \\
ph-calib  & 0.08 & 0.08 & 0.65 & 1.00 & 555 & 530 & 295 \\
\midrule
\multicolumn{8}{l}{\emph{TL;DR}} \\
eprocess  & 0.03 & 0.83 & 1.00 & 1.00 & 555 & 205 & 140 \\
naive-z   & 0.67 & 0.83 & 0.83 & 0.83 & 168 & 72  & 50  \\
ph-calib  & 0.03 & 0.40 & 1.00 & 1.00 & 550 & 248 & 165 \\
\bottomrule
\end{tabular}
\end{table}

\subsection{Generalization: every cell of both datasets}\label{sec:exp-generalize}

\paragraph{Every topic--rubric cell.} Running the featured config across all $28$ runnable
HelpSteer2 topic$\times$rubric cells (two near-ceiling coherence cells are skipped --- their
regressed-blind pools are empty), judge detection is $1.00$ in $27/28$ cells ($0.97$ in
factual\_qa:coherence) at latency $98$--$128$ items, judge$\to$system is $0.00$ in $27/28$
($0.03$ in that same cell), and the false-judge rate is at most $0.10$. On TL;DR's $19/20$
runnable subreddit$\times$axis cells the picture is the same: judge detection $0.85$--$0.97$
at latency $118$--$145$ in every cell, judge$\to$system $\leq 0.15$ at $\guard = 50$, and
false-judge $\leq 0.10 = \alphajudge$. The C2 properties are cell-independent by construction
--- the anchor process never looks at the main-process cell. Only C1 system power varies
(range $0.00$--$1.00$ across HelpSteer2 cells, $0.40$--$1.00$ on TL;DR, weakest where
$\cheap$ is most generous; missed system regressions fail safe to \texttt{none}, never a
false \texttt{judge}). The full per-cell table is in \cref{app:cells}.

\paragraph{Cross-dataset reads.} The exhibit-by-exhibit replication tests the theory in two
distinct ways: the \emph{invariances} must transfer, and the \emph{differences} must be the
ones the propositions predict. Both happen. On the invariance side, the same real version
bump shows the same $+0.03$ to $+0.07$ per-rubric signature in both domains
(\cref{tab:judgegap}) --- a property of the bump, not the dataset --- turning the lenient
case from a case study into a property of the version bump itself; and each dataset has one
low-resolution rubric (verbosity on HelpSteer2, coherence on TL;DR) with identical
phenomenology --- weak $\cheap$ correlation, near-ceiling humans, weakest judge--human
agreement --- a property of rubric design, not of the monitor. On the prediction side, the
domains differ in exactly one relevant input --- the strict-prompt shift is larger on
summaries --- and \cref{prop:race} says a larger gap shift makes the anchors stochastically
faster, so attribution should improve everywhere that race appears. It does, three times
over, with nothing re-tuned: anchor latency drops ($205$ versus $370$), the guard sweep
closes to zero spill at $\guard = 300$ (\cref{tab:guard}), and the co-provisioning band moves
toward slower anchor rates (\cref{fig:frontier}). These are out-of-sample structural
predictions coming true, not robustness checks. The one quantity that is \emph{not}
predicted to transfer, C1 system power, indeed varies for its own reason: it is
\emph{higher} on TL;DR ($56/60$ versus $41/60$, \cref{tab:confusion}) because all four axes
are informative for $\cheap$, unlike HelpSteer2's complexity and verbosity.

\section{Discussion and Limitations}\label{sec:discussion}


The construction buys a sharp guarantee --- judge drift is never silently charged to the
system --- at the price of several scoped assumptions, which we state plainly.

\paragraph{Anchor staleness.} Identification (\cref{prop:identification}) assumes the
frozen human labels $\human$ remain a valid reference. If the anchor items' \emph{true}
quality standard drifts conceptually --- the meaning of a rubric shifts, say --- the anchor
process still correctly reports judge-versus-anchor disagreement, but its interpretation as
\emph{judge} drift weakens, since the disagreement could now be the world moving under a
fixed yardstick. A periodic anchor refresh, re-labeling against the current standard,
re-establishes the baseline.

\paragraph{Pooled anchor family.} The anchor e-process aggregates over rubric$\times$direction
only; it carries no topic dimension. This is a deliberate power/budget choice: judge drift is
plausibly global, and a per-topic anchor family would split the $k$ anchors across cells and
multiply the Bonferroni factor by $K$. The cost is reduced power against \emph{topic-localized}
judge drift --- a judge that changed only on, say, code-heavy items. Topic stratification lives
entirely on the system-process side.

\paragraph{\texttt{both} is reported as \texttt{judge}, by design.} When both families fire
within the guard window the verdict is \texttt{judge} (\cref{eq:attribution}): a contaminated
$\strong$ invalidates the main alarm, so we do not attempt simultaneous attribution of a
co-occurring system regression. The honest reading is that a confirmed judge change suspends
trust in the main verdict until the judge is re-anchored.

\paragraph{No optimality.} We claim validity and identification, not efficiency. The plug-in
bet $\lambda$ approximates the log-optimal bet rather than realizing it, and the attribution
rule is not claimed to minimize misattribution at a fixed anchor budget. The race monotonicities
of \cref{prop:race} say which direction each knob moves the spill, not that the knobs are set
optimally.

\paragraph{System power is cell-dependent, and the horizon is single.} The disentanglement
properties are cell-independent by construction, but the main monitor's \emph{system} power is
not: it is near zero on the cheap-judge-blind rubrics (complexity, verbosity) and on the math
cells, where the regressed-blind pool is empty or the cheap judge is uninformative. A missed
system regression fails safe to \texttt{none} --- never to a false \texttt{judge}. The
experiments use a single monitoring horizon ($N = 1200$, change point at $300$) and a
representative token-cost model (input/output token counts at measured prices), which ignores
topic-assignment overhead; absolute cost fractions should be read as representative, not
universal.

\paragraph{The honest cost menu.} There is no single price. The system traces a menu against
strong-evaluating every item ($1.0$): a cheap-but-deaf configuration (escalation $\pi_0$ with
sparse rate-$20$ anchors) monitors at $\approx 0.21$ but misses most system regressions at the
hard cell; the featured powered-and-guarded configuration (fixed@$0.3$ main with rate-$5$
anchors) costs $\approx 0.64$ and catches system drift with zero misattribution. Notably, the
anchor budget is the \emph{smaller} part of the increment from cheap to powered: of the $0.43$
cost-fraction gap, the anchor share is only $0.20$ --- most of the added cost buys main-process
power, not anchor coverage.

\paragraph{A practical recipe.} These limitations compose into a deployment order, which is the
payoff for a practitioner. First, pick the main configuration on the cost/power frontier of
\cref{sec:setting} --- this fixes the system power you can afford and, with it, the
contaminated-fire latency the anchors must beat. Second, size the anchor budget $(k$, interleave
rate$)$ so that anchor detection latency, at the drift magnitudes you actually care about, comes
in \emph{under} the main's contaminated-fire latency --- the co-provisioning band of
\cref{sec:exp-coprovision}, not a bolted-on afterthought. Third, widen the guard $\guard$ to
absorb the residual spill: it is the cheapest knob, buying attribution correctness with latency
rather than dollars, and (\cref{prop:race}) it can never manufacture a false \texttt{judge}
verdict.

\section{Conclusion}\label{sec:conclusion}

Every team that monitors an LLM product with an LLM judge eventually hits the same question:
when the dashboard turns red, is it the product or the judge? We have shown that a small,
fixed set of human-labeled anchors, re-scored by the current judge and watched by one extra
anytime-valid e-process, answers it --- separating system drift from judge drift into an
explicit three-state verdict with three transparent
knobs (anchor size $k$, interleave rate $1/\rho$, and guard width $\guard$) and a guarantee
that judge drift is never charged to the system beyond $\alphajudge$. The construction is
validated not only on synthetic drift but on two \emph{real} judge changes --- a silent version
bump and a strict-prompt policy update --- across two domains, with nothing re-tuned between
them; where the domains differ, the differences are themselves predicted by the attribution
race (a larger judge shift yields faster anchors, perfect attribution, and a co-provisioning
band shifted toward slower anchor rates). Three directions remain open: principled anchor-refresh policies that keep $\human$ a
valid reference as standards evolve; per-stratum anchor families that recover power against
topic-localized judge drift; and optimal betting and attribution rules that minimize
misattribution at a fixed anchor budget.

\bibliographystyle{plainnat}
\bibliography{refs}
\appendix

\section{Proofs}\label{app:proofs}

We prove Propositions~\ref{prop:validity}--\ref{prop:orthogonality} as stated in
\cref{sec:anchor}, in the same notation. Throughout, ``past'' at position $i$ denotes the
$\sigma$-algebra $\mathcal{F}_{i-1}$ generated by the anchor observations
$\gapstat_1, \dots, \gapstat_{i-1}$ (and, for the main process, the stream and acquisition
history up to $i-1$). All validity claims are inherited from standard
test-(super)martingale results; the contribution is the construction they apply to, not new
probability theory.

\begin{proof}[Proof of \cref{prop:validity} (anytime-validity)]
Fix a rubric--direction pair $(r,d)$ and write $e_i = 1 + \lambda_i\, s_d\,(\gapbase[r] -
\gapstat_i[r])$ with $s_{\text{below}} = +1$, $s_{\text{above}} = -1$
(\cref{eq:anchor-evalue}).

\emph{Nonnegativity.} By \cref{eq:gap}, $\gapstat_i[r] \in [0,1]$ and $\gapbase[r] \in
(0,1)$, so the signed gap $s_d\,(\gapbase[r] - \gapstat_i[r])$ lies in $[-1, 1]$. The bet is
constrained to $\lambda_i \in [0, \lambda_{\max}(\gapbase[r], d)]$, where the directional cap
$\lambda_{\max}(\gapbase[r], d)$ is exactly the largest $\lambda \geq 0$ for which $1 + \lambda
\cdot s_d\,(\gapbase[r] - g)$ stays nonnegative for every attainable $g \in [0,1]$ (for $d =
\text{below}$ the binding case is $g = 1$, giving $\lambda_{\max} = 1/(1-\gapbase[r])$; for $d
= \text{above}$, $g = 0$, giving $\lambda_{\max} = 1/\gapbase[r]$). Hence $e_i \geq 0$ for all
$i$.

\emph{Conditional mean.} The bet $\lambda_i$ is predictable --- a function of
$\gapstat_1, \dots, \gapstat_{i-1}$ only (in code the plug-in $\lambda \approx
\operatorname{mean}(d)/\operatorname{mean}(d^2)$ computed from past gaps). Under
$H_0^{\mathrm{anc}}$ we have $\mathbb{E}[\gapstat_i[r] \mid \mathcal{F}_{i-1}] = \gapbase[r]$,
so
\[
  \mathbb{E}[e_i \mid \mathcal{F}_{i-1}]
  = 1 + \lambda_i\, s_d\,\big(\gapbase[r] - \mathbb{E}[\gapstat_i[r] \mid \mathcal{F}_{i-1}]\big)
  = 1 + \lambda_i\, s_d \cdot 0 = 1,
\]
using predictability of $\lambda_i$ to pull it outside the conditional expectation.

\emph{Supermartingale.} Therefore $\mathbb{E}[\wealth_n^{(r,d)} \mid \mathcal{F}_{n-1}] =
\wealth_{n-1}^{(r,d)} \cdot \mathbb{E}[e_n \mid \mathcal{F}_{n-1}] = \wealth_{n-1}^{(r,d)}$,
with $\wealth_0^{(r,d)} = 1$ (empty product). Each $(\wealth_n^{(r,d)})_n$ is a nonnegative
martingale, hence a nonnegative supermartingale.

\emph{Crossing bound and union.} Ville's inequality \citep{ville1939} for a nonnegative
supermartingale with unit initial value gives, for any level $\beta > 0$,
$\mathbb{P}(\exists n: \wealth_n^{(r,d)} \geq \beta) \leq 1/\beta$. Taking $\beta = 2R /
\alphajudge$ bounds each pair's ever-crossing probability by $\alphajudge / 2R$. A union bound
over the $2R$ rubric--direction pairs gives
\[
  \mathbb{P}\Big(\exists n: \max_{r,d}\wealth_n^{(r,d)} \geq 2R/\alphajudge\Big)
  \;\leq\; 2R \cdot \frac{\alphajudge}{2R} = \alphajudge .
\]
Ville's inequality is time-uniform, so the bound holds at any data-dependent stopping time.
\end{proof}

\begin{remark}
The bound needs no distributional assumptions beyond boundedness: the gaps need not be
Gaussian, i.i.d.\ across rubrics, or stationary in variance. It also needs no held-out
calibration streams and no fixed horizon, in contrast to the classical change detectors of
\cref{sec:exp-baselines}, which set a threshold against $H_0$ streams and are calibrated only
at that horizon.
\end{remark}

\begin{remark}[Exactness of the baseline]
Anchors are sampled \emph{with replacement} from the fixed set $\anchorset$, so under
$H_0^{\mathrm{anc}}$ the conditional-mean condition $\mathbb{E}[\gapstat_i[r] \mid
\mathcal{F}_{i-1}] = \gapbase[r]$ holds by construction of $\gapbase$ up to the finite-sample
error of $\gapbase$ itself. Because $\gapbase[r]$ is defined as the \emph{exact} mean rescaled
gap of the time-0 judge over all of $\anchorset$, and the time-0 judge scores are cached, this
error is zero by construction: the unchanged judge re-scoring a uniformly drawn anchor has gap
mean exactly $\gapbase[r]$.
\end{remark}

\begin{proof}[Proof of \cref{prop:identification} (identification)]
Each anchor observation at any position is a deterministic function of (a) the fixed items
$\anchorset$, (b) their frozen labels $\human$, and (c) the current judge, through
\cref{eq:gap}. System drift --- any change in the distribution of \emph{new} stream items or
in their true quality --- alters none of (a)--(c), so it leaves the anchor gap distribution
unchanged. Hence $H_0^{\mathrm{anc}}$ can be violated only by a change in the judge.

\emph{(a)} Under pure system drift the judge is unchanged, so $H_0^{\mathrm{anc}}$ holds. A
\texttt{judge} verdict (\cref{eq:attribution}) requires $\tau_{\mathrm{anc}} < \infty$, i.e.\
the anchor family fires; \cref{prop:validity} bounds the probability of that event by
$\alphajudge$ at any horizon, and this bound does not involve $\guard$. Therefore
$\mathbb{P}(\text{verdict} = \texttt{judge} \mid \text{pure system drift}) \leq \alphajudge$,
uniformly in $\guard$.

\emph{(b)} Under no drift at all, both nulls hold. The event $\{\text{verdict} \neq
\texttt{none}\}$ requires at least one family to fire, i.e.\ $\{\tau_{\mathrm{sys}} < \infty\}
\cup \{\tau_{\mathrm{anc}} < \infty\}$. By the main-process guarantee (\cref{sec:setting},
inherited from \citet[Thm.~2.1]{csillag2025ppe} with predictable $\pi$ and Ville's inequality)
$\mathbb{P}(\tau_{\mathrm{sys}} < \infty) \leq \alphasys$, and by \cref{prop:validity}
$\mathbb{P}(\tau_{\mathrm{anc}} < \infty) \leq \alphajudge$. A union bound gives
$\mathbb{P}(\text{verdict} \neq \texttt{none}) \leq \alphasys + \alphajudge$; the two budgets
are separate by design.
\end{proof}

\begin{proof}[Proof of \cref{prop:race} (the attribution race)]
Suppose judge drift has occurred (with or without simultaneous system drift). By the verdict
rule \cref{eq:attribution}, $\text{verdict} = \texttt{system}$ holds exactly when
$\tau_{\mathrm{sys}} < \infty$ and ($\tau_{\mathrm{anc}} = \infty$ or $\tau_{\mathrm{anc}} >
\tau_{\mathrm{sys}} + \guard$), which is precisely the event $\{\tau_{\mathrm{sys}} + \guard <
\tau_{\mathrm{anc}}\}$ (adopting the convention $\tau_{\mathrm{sys}} + \guard < \infty =
\tau_{\mathrm{anc}}$ when the anchors never fire). Hence
\[
  \mathbb{P}(\text{verdict} = \texttt{system} \mid \text{judge drift})
  = \mathbb{P}(\tau_{\mathrm{sys}} + \guard < \tau_{\mathrm{anc}}).
\]

\emph{Monotonicity in $\guard$.} For fixed sample paths of the two stopping times, the
indicator $\mathbf{1}\{\tau_{\mathrm{sys}} + \guard < \tau_{\mathrm{anc}}\}$ is non-increasing
in $\guard$ pointwise; taking expectations preserves the order, so the misattribution
probability is monotonically non-increasing in $\guard$.

\emph{Monotonicity in the anchor budget.} Increasing the interleave rate $1/\rho$ or the
anchor set size $k$ supplies more anchor evidence per position, which stochastically decreases
$\tau_{\mathrm{anc}}$ (first-order stochastic dominance: more accumulated wealth per step
reaches the threshold $2R/\alphajudge$ no later). A stochastically smaller $\tau_{\mathrm{anc}}$
makes the event $\{\tau_{\mathrm{sys}} + \guard < \tau_{\mathrm{anc}}\}$ less likely, so the
misattribution probability is non-increasing in $1/\rho$ and in $k$.

\emph{Monotonicity in main power.} Increasing the main process's sampling power (more strong
calls on the contaminated $\strong$) stochastically decreases $\tau_{\mathrm{sys}}$, by the
same first-order argument applied to the main e-process: more strong labels per step grow the
contaminated wealth no slower toward its threshold. A stochastically smaller
$\tau_{\mathrm{sys}}$ makes the misattribution event more likely, so the probability is
non-decreasing in main power.

These are stochastic-dominance arguments at the level of rigor of the formal-statements
source; we do not claim a sharper coupling.
\end{proof}

\begin{remark}
The race involves no validity trade-off. Raising $\guard$ delays a \texttt{system} attribution
but cannot create false \texttt{judge} verdicts beyond the $\alphajudge$ of
\cref{prop:identification}(a), because the anchor bound of \cref{prop:validity} holds at
\emph{any} horizon --- this time-uniformity is exactly what makes the guard window free.
Detection is also untouched: the total event $\{\text{drift flagged at all}\} =
\{\tau_{\mathrm{sys}} \wedge \tau_{\mathrm{anc}} < \infty\}$ does not depend on $\guard$; the
rule only decides the label. Empirically (\cref{sec:exp-harsh}) the spill shrinks monotonically
in $\guard$ ($29 \to 18 \to 10$ of $120$ at $\guard = 50/150/300$) and in the anchor rate, and
grows with main power ($0.03 \to 0.33$ at rate-$5$ across the main ladder).
\end{remark}

\begin{proof}[Proof of \cref{prop:orthogonality} (orthogonality)]
The anchor observations $\gapstat_i$ are a deterministic function of $(\anchorset, \human,
\text{current judge})$ and the fixed interleave schedule (\cref{eq:gap}), none of which depend
on the main acquisition rule (policy, budget, or $\pi$). Therefore the law of the anchor
process --- and in particular $\tau_{\mathrm{anc}}$ --- does not depend on the main-process
configuration. Symmetrically, the anchor re-judgements are out-of-band calls on held-out items
that, by the anchor-set definition (\cref{sec:anchor}), are excluded from all main-process
calibration and never enter the main e-values; the anchor process never reads stream items.
Hence the main process's distribution does not depend on the anchor interleave. Given the
judge state, the two stopping times are functions of disjoint randomness, and the design splits
cleanly: the anchor budget $(k, \rho)$ and guard $\guard$ govern judge-drift latency and the
race; the main configuration governs system power; and $\alphajudge, \alphasys$ are separate
Bonferroni budgets.
\end{proof}

\begin{remark}
Orthogonality holds for the two \emph{processes}; the \emph{verdict} still couples them through
the race of \cref{prop:race} (the comparison $\tau_{\mathrm{anc}}$ vs $\tau_{\mathrm{sys}} +
\guard$ reads both stopping times). Result~5 of \cref{sec:exp-coprovision} verifies the
process-level orthogonality empirically: judge latency ($830/695/430$) is constant across all
five main configs, and each main's system power is constant across anchor rates.
\end{remark}

\paragraph{What is not claimed.} Two boundaries, transcribed from the formal-statements source.
\emph{No optimality}: the plug-in $\lambda$ approximates the log-optimal bet, and the
attribution rule is not claimed to minimize misattribution at a fixed budget. \emph{Anchor
staleness}: \cref{prop:identification} assumes the human labels $\human$ remain a valid
reference; if the anchor items' \emph{true} quality standard drifts conceptually (e.g.\ the
rubric meaning changes), the anchor process correctly reports judge-vs-anchor disagreement, but
its reading as judge drift weakens until a periodic anchor refresh re-establishes the baseline.
Finally, the anchor family is pooled across strata (rubric$\times$direction only), so a judge
change confined to one stratum is detected at reduced power; topic stratification lives entirely
on the system-process side.

\section{The Cost-Aware Foundation in Detail}\label{app:c1}

This appendix gives the full cost-frontier table, the ten-method prior-work comparison, and
the trigger ablations summarized in \cref{sec:setting}. All numbers are for the semi-synthetic
localized blind-spot drift in the helpfulness$\times$creative cell of HelpSteer2 (stream $N =
1200$, change at $300$), averaged over $60$ Monte-Carlo seeds, at $\alphasys = 0.1$.

\paragraph{Cost model recap.} Each judge call costs $\text{(input tokens)}\cdot\text{(input
price)} + \text{(output tokens)}\cdot\text{(output price)}$ at representative counts of
$\approx 600$ input and $\approx 30$ output tokens, giving $\$0.000195$ per cheap call and
$\$0.00156$ per strong call. The cheap judge runs on all $L$ processed items and the strong
judge on $S$ sampled items, so relative to strong-evaluating every item the
$\text{cost-fraction} = (L c_{\cheap} + S c_{\strong}) / (L c_{\strong}) = 0.125 + S/L$, with
full evaluation ($\pi \equiv 1$) equal to $1.0$ by definition. The $0.125$ is the irreducible
cheap floor and $S/L$ is the realized strong-sampling rate.

\subsection{The cost--reliability frontier}

\cref{tab:frontier} reports the per-stratum monitor across both acquisition modes.
Latency is the median steps from the change point over uncensored reps; censor is the fraction
of reps that never alarmed within the horizon.

\begin{table}[h]
\centering\small
\caption{Cost--reliability frontier for the per-stratum monitor (helpfulness$\times$creative,
$y_{\mathrm{low}} = 0.25$, 60 seeds). Latency = median steps to detection over uncensored reps;
censor = fraction never alarming within the horizon ($N = 1200$, change at $300$).}
\label{tab:frontier}
\begin{tabular}{l ccccc}
\toprule
approach & latency & censor & cost-fraction & false-alarm & in-control cost \\
\midrule
escalate $\pi_0 = 0.01$            & ---  & 1.00 & ---           & 0.000 & 0.142 \\
escalate $\pi_0 = 0.02$            & 750  & 0.95 & \textbf{0.21} & 0.000 & 0.154 \\
escalate $\pi_0 = 0.05$            & 710  & 0.82 & \textbf{0.23} & 0.000 & 0.184 \\
escalate $\pi_0 = 0.1$             & 706  & 0.53 & \textbf{0.28} & 0.000 & 0.236 \\
fixed uniform @0.1                 & 789  & 0.98 & 0.24          & 0.000 & 0.226 \\
fixed @0.2 (uni/strat/bandit)      & ${\sim}660$ & 0.50 & 0.33   & 0.000 & 0.327 \\
fixed @0.3 (uni/strat/bandit)      & ${\sim}585$ & 0.28 & 0.43   & 0.000 & 0.428 \\
\textbf{full strong-eval ($\pi = 1$)} & \textbf{327} & 0.02 & \textbf{1.00} & 0.000 & 1.125 \\
\bottomrule
\end{tabular}
\end{table}

Every detector catches the drift at cost-fraction $0.21$--$0.43$, about one third of full
strong evaluation, with false-alarm $0.000$ across all $16$ policy points (empirically
anytime-valid). The two modes trace a frontier: escalation is cheapest ($0.21$--$0.28$),
sitting at the $0.125 + \pi_0$ floor in-control --- near-zero strong calls when nothing is
wrong; fixed budgets buy reliability (@0.2 $\to$ censor $0.50$, @0.3 $\to$ censor $0.28$); full
evaluation is fastest (latency $327$) at full price. Escalation $\pi_0 = 0.1$ (cost $0.28$,
censor $0.53$) is about as reliable as fixed @0.2 (cost $0.33$, censor $0.50$) but cheaper,
missing the strict censor-$\leq 0.5$ cut by one rep. This drift is intrinsically hard (even
full-eval censors $2\%$), which compresses escalation's margin; a stronger or cleaner drift, or
a higher $\pi_0$, lets escalation dominate the frontier outright.

\subsection{The ten-method prior-work comparison}

We place the four prior-work families in one harness on one frontier, calibrated to the same
false-alarm level. \cref{tab:roster} lists what each method sees and acquires; ``sees $\cheap$''
means it runs the cheap judge on every item, and the cost-fraction charges each method only for
what it calls.

\begin{table}[h]
\centering\small
\setlength{\tabcolsep}{4pt}
\caption{The ten-method roster. ``isolates'' names the design axis each method isolates.}
\label{tab:roster}
\begin{tabular}{@{}l l >{\raggedright\arraybackslash}p{2.65cm} c l c >{\raggedright\arraybackslash}p{2.45cm}@{}}
\toprule
group & \# & method & sees $\cheap$ & acquires $\strong$ & cost-frac & isolates \\
\midrule
\multirow{3}{*}{A. cheap-only} & A1 & same e-process on $\cheap$        & \checkmark & never        & 0.125       & cheap carries no drift info \\
                & A2 & Page--Hinkley on $\cheap$         & \checkmark & never        & 0.125       & concept-drift family \\
                & A3 & Shiryaev--Roberts on $\cheap$     & \checkmark & never        & 0.125       & LR/e-process detector \\
\addlinespace
B. label-only   & B1 & e-process on $\strong$, no PPI    & ---        & fixed $\tau$ & $\tau$      & value of the cheap proxy \\
\addlinespace
C. global PPI   & C1 & per-rubric PPI, no strata         & \checkmark & fixed $\tau$ & $0.125{+}\tau$ & value of stratification \\
\addlinespace
\multirow{4}{*}{\shortstack[l]{D. per-stratum\\PPI}} & D1 & fixed-budget $\tau$ & \checkmark & fixed $\tau$ & $0.125{+}\tau$ & --- (fixed-budget foil) \\
                & D2 & covariate: $\pi{\uparrow}$ as $\cheap{\to}q_0$ & \checkmark & covariate $\cheap$ & $0.125{+}S/L$ & covariate vs e-wealth \\
                & D3 & Csillag App B.2: $\pi \propto$ pred.\ $e$-growth & \checkmark & covariate $\cheap$ & $0.125{+}S/L$ & the published active-PPI rule \\
                & D4 & escalation (ours): $\pi{\uparrow}$ with $e$-wealth & \checkmark & past $e$-evidence & $0.125{+}S/L$ & --- (our trigger) \\
\addlinespace
E. reference    & E1 & full strong-eval $\pi = 1$        & ---        & every item   & 1.0         & speed ceiling \\
\bottomrule
\end{tabular}
\end{table}

The prior-work mapping: A2/A3 are the concept-drift family (Page--Hinkley
\citep{page1954}/Shiryaev--Roberts \citep{shiryaev1963}); B1 is label-only SAVI; C1 is the
global (un-stratified) PPI e-process of \citet{csillag2025ppe}; D2 is Zrnic--Cand\`es
uncertainty-triggered active inference \citep{zrnic2024active}; D3 is the Csillag App~B.2
active-PPI rule; D4 is ours; E1 is the naive ground-truth-everything monitor. Every e-process
method is false-alarm-valid by construction (Bonferroni), and all four per-stratum PPI rules
(D1--D4) share the \emph{same} per-item bet cap $\lambda_{\max}(\pi_i)$ on their realized
(predictable) acquisition probability, so no method gets a betting head-start.

\paragraph{The acquisition-rule headline.} \cref{tab:triggers} ranks the four per-stratum PPI
rules, which differ \emph{only} in the acquisition rule, against drift strength. ``Detects''
means censor $\leq 0.5$ and false-alarm $\leq \alpha$. Because the strong label $\strong$ lives
on the HelpSteer2 5-level grid $\{0, 0.25, 0.5, 0.75, 1.0\}$, the four-point drift sweep
($y_{\mathrm{low}} \in \{0.15, 0.25, 0.35, 0.45\}$) collapses to two effective magnitudes in
this cell: $y_{\mathrm{low}} = 0.15$ draws a pool of one extreme item (severe), while
$y_{\mathrm{low}} \in \{0.25, 0.35, 0.45\}$ draw the \emph{identical} 21-item pool (mild,
byte-identical in the raw JSON). We report the comparison as severe vs mild.

\begin{table}[h]
\centering\small
\setlength{\tabcolsep}{4pt}
\caption{Acquisition-rule ranking vs drift strength (per-stratum PPI; D1--D4 differ only in the
trigger, sharing the per-item bet cap). Severe = $y_{\mathrm{low}} = 0.15$; mild =
$y_{\mathrm{low}} \geq 0.25$.}
\label{tab:triggers}
\begin{tabular}{@{}l p{0.36\textwidth} p{0.36\textwidth}@{}}
\toprule
& severe ($y_{\mathrm{low}} = 0.15$) & mild ($y_{\mathrm{low}} \geq 0.25$) \\
\midrule
D1 fixed-budget        & best cheap detector: @0.2 censor 0.12 (cost 0.33), @0.3 censor 0.05 (cost 0.43) & @0.2 censor 0.50 (cost 0.33), @0.3 censor 0.27 (cost 0.43) \\
D2 covariate-boundary  & detects fast but by oversampling: censor 0.00 at every $\pi_0$, latency ${\sim}185$, cost ${\sim}0.87$ & detects, still by oversampling: censor 0.00 at every $\pi_0$, latency ${\sim}430$, cost ${\sim}0.88$ \\
D3 covariate-egrowth   & weak: $\pi_0 = 0.1$ censor 0.80 (cost 0.24), else censor 1.00 & fails (censor $\geq 0.95$ at every $\pi_0$) \\
D4 escalation (ours)   & strong and cheap: $\pi_0 = 0.1$ censor 0.22 (cost \textbf{0.28}) & ties fixed-budget: $\pi_0 = 0.1$ censor 0.53 (cost 0.28) $\approx$ D1@0.2 censor 0.50 (cost 0.33) \\
\bottomrule
\end{tabular}
\end{table}

The covariate-triggered prior-art rules do not fail to \emph{detect} --- they fail to detect
\emph{efficiently}, each for a distinct reason. D2 (boundary) detects on both drifts but only
by oversampling to cost $\approx 0.87$--$0.89$ (near full-eval): its ``label near the decision
boundary'' heuristic fires near \emph{every} rubric's bar across the $30$ cells, so $\pi$ is
pushed high regardless of where the drift is. D3 (the published Csillag App~B.2 rule) is inert:
its predicted-$e$-growth trigger reads identically zero because the cheap judge overrates
\emph{every} healthy bar ($q_0 - \cheap < 0$ on all five rubrics: helpfulness $-0.18$,
complexity $-0.23$, verbosity $-0.27$), so it degenerates to fixed-budget at rate $\pi_0$ and
fails on the mild drift. D4 (ours), keying on accumulated $e$-evidence rather than the blind
cheap covariate, detects at cost $0.28$ --- roughly one third of D2's --- by spending only once
evidence accrues; on the faint mild drift it merely \emph{ties} a matched fixed budget
(D4@0.1 censor 0.53 cost 0.28 $\approx$ D1@0.2 censor 0.50 cost 0.33), exactly the regime of
\citet{sfyraki2026adaptive} where an adaptive trigger buys essentially nothing over fixed-$\tau$
for prediction-powered mean estimation. We report the tie as-is. The central result, restated:
keying acquisition on the blind cheap covariate is either wasteful (D2 oversamples to
$\approx 0.88$) or inert (D3 reads zero under the overrating bias); keying on accumulated
$e$-evidence (D4) detects at one third of that cost.

\paragraph{Two further contrasts.} \emph{Stratification vs global}: per-(rubric$\times$topic)
stratification is decisive. At every matched budget, global PPI (C1) censors strictly worse
than per-stratum fixed-budget (D1) --- severe @0.2 global $0.68$ vs stratified $0.12$, @0.3
$0.58$ vs $0.05$; on the mild drift global PPI censors $1.00$ at all three budgets while
stratified detects at @0.2 ($0.50$) and @0.3 ($0.27$). A localized blind-spot drift is diluted
to invisibility when pooled. \emph{Classic detectors do not transfer}: the two concept-drift
detectors are threshold-calibrated to family-wise false-alarm $\approx \alpha$ on a held-out,
disjoint $H_0$ seed block, yet off that sample held-out Page--Hinkley records false-alarm
$0.117$ (just over $\alpha$) and held-out Shiryaev--Roberts $0.267$ (well over $\alpha$) --- the
realized maximum false-alarm across all methods and settings, $0.267$, is driven entirely by
the classic SR detector, while the e-process spine holds at false-alarm $\leq 0.017$ (the lone
non-zero e-process value is covariate-boundary $\pi_0 = 0.1$ at $0.017$, one $H_0$ alarm in
$60$, within $\alpha$). The cheap proxy alone never sees this drift: all three cheap-only
methods censor near $1.0$ everywhere (A1 ablation $1.00$, A3 SR $1.00$, A2 Page--Hinkley $0.98$
severe / $0.92$ mild), by construction of the blind-spot drift.

\section{All-Cells Results}\label{app:cells}

\cref{tab:cells} gives the full per-(topic$\times$rubric) cell table summarized in
\cref{sec:exp-generalize}, for the featured configuration (fixed@0.3 main, $k = 200$ anchors at
rate $1/5$, guard $\guard = 50$, synthetic shift $-0.25$, $20$ seeds) on HelpSteer2. The two
near-ceiling coherence cells coding:coherence and math:coherence are skipped: their
regressed-blind pool ($\strong \leq 0.5 \wedge \cheap \geq 0.6$) is empty because coherence
sits at ${\sim}0.92$. Columns: judge detection rate, median judge-drift latency (items),
judge$\to$system spill rate, false-judge rate, and system detection rate.

The C2 properties are cell-independent by construction --- the anchor process never looks at
the main-process cell. Judge detection is $1.00$ in $27/28$ cells (and $0.97$ in
factual\_qa:coherence), at latency $98$--$128$ items everywhere; judge$\to$system is $0.00$ in
$27/28$ ($0.03$ in that same cell); and the false-judge rate is at most $0.10$, under
$\alphajudge$. Only the system detection rate varies, a pure C1 sensitivity property:
correctness and coherence cells detect at $0.80$--$1.00$, helpfulness in the middle
($0.45$--$0.95$), and complexity/verbosity $\approx 0$ everywhere (the rubrics where $\cheap$ is
blind and the regressed pools are thin), with all math cells $\approx 0$. Every missed system
regression fails safe to \texttt{none}, never a false \texttt{judge}.

\begin{table}[h]
\centering\footnotesize
\setlength{\tabcolsep}{2pt}
\caption{Every topic$\times$rubric cell at the featured config (HelpSteer2, fixed@0.3 main,
$k{=}200$/rate-$5$ anchors, $\guard{=}50$, synthetic $-0.25$, 20 seeds). 28 runnable cells in
two column blocks; coding:coherence and math:coherence are skipped (empty regressed-blind pool).
jdet = judge detection, jlat = median judge latency (items), j$\to$s = judge$\to$system spill,
fj = false-judge, sdet = system detection.}
\label{tab:cells}
\begin{tabular}{@{}l ccccc c l ccccc@{}}
\toprule
cell & jdet & jlat & j$\to$s & fj & sdet & & cell & jdet & jlat & j$\to$s & fj & sdet \\
\midrule
coding:complexity & 1.00 & 122 & 0.00 & 0.00 & 0.00 & & math:complexity & 1.00 & 120 & 0.00 & 0.05 & 0.00 \\
coding:correctness & 1.00 & 122 & 0.00 & 0.00 & 0.80 & & math:correctness & 1.00 & 122 & 0.00 & 0.10 & 0.00 \\
coding:helpfulness & 1.00 & 125 & 0.00 & 0.00 & 0.60 & & math:helpfulness & 1.00 & 120 & 0.00 & 0.05 & 0.00 \\
coding:verbosity & 1.00 & 125 & 0.00 & 0.00 & 0.05 & & math:verbosity & 1.00 & 120 & 0.00 & 0.05 & 0.00 \\
creative:coherence & 1.00 & 110 & 0.00 & 0.00 & 1.00 & & other:coherence & 1.00 & 112 & 0.00 & 0.00 & 1.00 \\
creative:complexity & 1.00 & 118 & 0.00 & 0.05 & 0.00 & & other:complexity & 1.00 & 125 & 0.00 & 0.05 & 0.00 \\
creative:correctness & 1.00 & 115 & 0.00 & 0.00 & 1.00 & & other:correctness & 1.00 & 118 & 0.00 & 0.00 & 1.00 \\
creative:helpfulness & 1.00 & 115 & 0.00 & 0.00 & 0.65 & & other:helpfulness & 1.00 & 115 & 0.00 & 0.00 & 0.60 \\
creative:verbosity & 1.00 & 110 & 0.00 & 0.00 & 0.00 & & other:verbosity & 1.00 & 128 & 0.00 & 0.00 & 0.05 \\
factual\_qa:coherence & 0.97 & 98 & 0.03 & 0.05 & 0.95 & & reasoning:coherence & 1.00 & 118 & 0.00 & 0.03 & 1.00 \\
factual\_qa:complexity & 1.00 & 122 & 0.00 & 0.00 & 0.00 & & reasoning:complexity & 1.00 & 115 & 0.00 & 0.00 & 0.00 \\
factual\_qa:correctness & 1.00 & 115 & 0.00 & 0.00 & 0.95 & & reasoning:correctness & 1.00 & 120 & 0.00 & 0.05 & 0.85 \\
factual\_qa:helpfulness & 1.00 & 108 & 0.00 & 0.05 & 0.95 & & reasoning:helpfulness & 1.00 & 120 & 0.00 & 0.05 & 0.45 \\
factual\_qa:verbosity & 1.00 & 118 & 0.00 & 0.05 & 0.25 & & reasoning:verbosity & 1.00 & 112 & 0.00 & 0.05 & 0.10 \\
\bottomrule
\end{tabular}
\end{table}

The TL;DR replication of this sweep (\cref{sec:exp-generalize}) runs $19/20$ cells
(AskReddit:coherence skipped, near-ceiling): judge detection $0.85$--$0.97$ at latency
$118$--$145$ in every cell, judge$\to$system $\leq 0.15$ at $\guard = 50$, false-judge
$\leq 0.10 = \alphajudge$, and system detection $0.40$--$1.00$ (weakest on coverage cells,
the rubric where $\cheap$ is most generous).

\section{Cheap-Judge Calibration}\label{app:calibration}

The cheap judge $\cheap$ (\texttt{gemini-3.1-flash-lite}) is run on every item as free side
information for the prediction-powered main process. \cref{tab:calibration} reports how much
signal it carries per rubric, from the fitted-calibration study: a per-rubric $\cheap \to
\strong$ correction (shift / affine / isotonic) is fit on the first $600$ HelpSteer2 items and
measured on the held-out next $600$.

\begin{table}[h]
\centering\small
\setlength{\tabcolsep}{3pt}
\caption{Cheap-judge calibration against the strong judge, held-out ($n = 600$ test). Bias
$= \cheap - \strong$; corr is Pearson; affine slope $a$ is the fitted $\cheap \to \strong$
gain; RMSE$\downarrow$ is the held-out RMSE reduction of the affine correction vs raw $\cheap$.}
\label{tab:calibration}
\begin{tabular}{@{}l ccccl@{}}
\toprule
rubric & $\cheap - \strong$ bias & corr$(\cheap, \strong)$ & affine slope $a$ & RMSE$\downarrow$ vs raw $\cheap$ & read \\
\midrule
helpfulness & $+0.13$ & 0.81 & 0.87 & $-19\%$ & informative, just generous \\
correctness & $+0.01$ & 0.81 & 0.80 & $-6\%$  & already accurate --- don't correct \\
coherence   & $-0.02$ & 0.63 & 0.60 & $-16\%$ & tracks \\
complexity  & $+0.16$ & 0.54 & 0.40 & $-54\%$ & biased but correctable \\
verbosity   & $+0.22$ & 0.23 & 0.28 & $-49\%$ & blind: $\cheap \approx$ uninformative \\
\bottomrule
\end{tabular}
\end{table}

The corrections transfer cleanly: held-out residual bias $\approx 0$ on every rubric, and
calibration RMSE $\approx$ test RMSE (no overfit). The affine fit $\approx$ isotonic fit (since
$\cheap$ lives on a 5-level grid), and a 1-parameter \emph{shift} already removes most of the
error; the blind-spot structure is stable across the split ($|\Delta\text{bias}| \leq 0.01$,
$|\Delta\text{blind\%}| \leq 5$ points). The takeaway: calibration removes the cheap judge's
systematic generosity but cannot create resolution it never had. The affine slope $a$ measures
how much $\cheap$ is worth: $a \approx 1$ (helpfulness) means keep its signal and just shift it;
$a \approx 0$ (verbosity) means $\cheap$ is essentially the base rate, so the strong judge is
required there. Helpfulness and complexity become good calibrated proxies; verbosity does not
--- which is exactly the rubric where the localized blind-spot drift of \cref{app:c1} hides, and
where the large judge-version gap of \cref{tab:judgegap} appears.

\section{Judge Prompts}\label{app:prompts}

The strong and cheap judges share a scoring prompt; only the model snapshot differs. Drift is
injected on the strong judge in two ways: a real version bump (re-judging with a smaller model)
needs no prompt change, while the real \emph{harsh} change re-judges with the same model under a
deliberately stricter scoring prompt. We quote both prompt versions verbatim from the registry
(\texttt{driftjudge/data/judges.py}); the rubric list and prompt/response are interpolated at
the marked fields.

\paragraph{Baseline scoring prompt (\texttt{v1}).}
\begin{quote}\small\ttfamily
You are a strict evaluator. Score the assistant RESPONSE to the PROMPT on each rubric using an
integer 0-4 (HelpSteer2 scale: 0 worst, 4 best). Reply with ONLY a JSON object mapping each
rubric name to its integer score.\\
Rubrics: \{rubrics\}\\[2pt]
PROMPT:\\
\{prompt\}\\[2pt]
RESPONSE:\\
\{response\}
\end{quote}

\paragraph{Harsh scoring prompt (\texttt{v2-strict}).} A deliberately harsher scoring policy on
the same scale and model, emulating a real silent scoring-policy update (C2's contamination
direction). The money clause is the strict tie-break:
\begin{quote}\small\ttfamily
You are an exacting, skeptical evaluator. Score the assistant RESPONSE to the PROMPT on each
rubric using an integer 0-4 (HelpSteer2 scale: 0 worst, 4 best). Be strict: reserve 4 for
flawless work, penalize every error, omission, or unsupported claim you notice, and when torn
between two adjacent scores always choose the lower. Reply with ONLY a JSON object mapping each
rubric name to its integer score.\\
Rubrics: \{rubrics\}\\[2pt]
PROMPT:\\
\{prompt\}\\[2pt]
RESPONSE:\\
\{response\}
\end{quote}

The TL;DR summarization domain uses parallel variants \texttt{v1-summ} and \texttt{v2-strict-summ}
(same instructions, with POST/SUMMARY framing and the four summarization axes substituted for the
rubric meanings). An unknown \texttt{prompt\_version} raises. The harsh re-judge has $3$
HelpSteer2 items (and $1$ TL;DR item) that the strong endpoint cannot score (content filter);
these fall back to the undrifted $\strong$, so the contaminated stream is $1495/1498$
(HelpSteer2) and $1497/1498$ (TL;DR).

\section{Reproducibility}\label{app:repro}

The pipeline has two phases. The \emph{data layer} is a one-time paid Gemini cache (the only
network/API phase): every item is scored once by the cheap judge, once by the strong judge, and
once per drift slice, and the human labels are joined in. All experiments afterward are
\emph{compute-only}: they read the cached scores from \texttt{artifacts/} and re-run the
monitor in simulation, so reproducing every number below requires no API access once the cache
exists. The cached artifacts are \emph{not committed to the source tree} (\texttt{artifacts/} is
gitignored); re-collecting scores from a later judge version would yield different numbers ---
which is precisely the paper's point. For exact public reproduction, the intended release vehicle
is a separate versioned artifact archive, not committed source files. Figures regenerate from the
checked-in computed outputs via \texttt{make figures}
(\texttt{scripts/paper\_figures.py}).

\paragraph{Datasets and licensing.} HelpSteer2 \citep{helpsteer2} is released under CC-BY-4.0
(\textcopyright\ NVIDIA); the OpenAI summarize-from-feedback feedback data is released under
OpenAI's modified MIT license; and the underlying Webis TLDR corpus is listed under CC-BY-4.0.
The reproducibility artifact redistributes only derived judge scores and human labels/metadata,
with attribution, not raw source text. Gemini scores are Gemini-API outputs.

\paragraph{Data layer (one-time, paid).}
{\small\begin{verbatim}
HF_HUB_OFFLINE=1 uv run --extra pipeline python scripts/add_human_labels.py
GEMINI_API_KEY=... uv run --extra pipeline python scripts/judge_drift_data.py \
    --drift-snapshot gemini-3.5-flash --limit 1498
GEMINI_API_KEY=... HF_HUB_OFFLINE=1 uv run --extra pipeline python \
    scripts/judge_drift_data.py \
    --drift-snapshot gemini-3.1-pro-preview --prompt-version v2-strict --limit 1498
\end{verbatim}}

\paragraph{HelpSteer2 experiments (compute-only).}
{\small\begin{verbatim}
# Result 1b -- high-power main (fixed @0.3), matched vs old anchor budget:
uv run python scripts/c2_anchor_experiment.py --seeds 60 --k 200 --anchor-rate 5 \
    --main-mode fixed --main-budget 0.3
uv run python scripts/c2_anchor_experiment.py --seeds 60 --k 50 --anchor-rate 20 \
    --main-mode fixed --main-budget 0.3
# Result 3 -- real lenient bump, featured fixed@0.3 main:
uv run python scripts/c2_anchor_experiment.py --seeds 60 --k 200 --anchor-rate 5 \
    --main-mode fixed --main-budget 0.3 --real-drift-snapshot gemini-3.5-flash
# Result 4 -- real harsh drift, guard sweep (data study first):
uv run python scripts/study_judge_version_gap.py \
    --snapshot gemini-3.1-pro-preview --prompt-version v2-strict
for W in 50 150 300; do \
  uv run python scripts/c2_anchor_experiment.py --seeds 60 --k 200 --anchor-rate 5 \
      --guard $W --main-mode fixed --main-budget 0.3 \
      --real-drift-snapshot gemini-3.1-pro-preview --drift-prompt-version v2-strict; done
# Result 5 -- co-provisioning frontier (real harsh drift):
uv run python scripts/c2_anchor_experiment.py --seeds 30 --k 200 --guard 50 --frontier \
    --real-drift-snapshot gemini-3.1-pro-preview --drift-prompt-version v2-strict
# Result 6 -- every cell at the featured config -> artifacts/c2_cells.json:
uv run python scripts/c2_anchor_experiment.py --seeds 20 --k 200 --anchor-rate 5 \
    --main-mode fixed --main-budget 0.3 --all-cells
# Result 7 -- baselines on the identical anchor stream:
uv run python scripts/c2_baseline_comparison.py --seeds 60
\end{verbatim}}

\paragraph{TL;DR second dataset.} The TL;DR data layer is a separate one-time paid run (raw
files fetched per \texttt{driftjudge/data/tldr.py}); every experiment reuses the same drivers
with \texttt{-{}-dataset tldr} (results land in \texttt{*\_tldr.json} artifacts, so the
HelpSteer2 result JSONs are untouched). The TL;DR halves of every table and figure mirror the
HelpSteer2 commands above:
{\small\begin{verbatim}
# data layer (paid, one-time, ~6k calls):
GEMINI_API_KEY=... uv run --extra pipeline python scripts/tldr_data.py --max-workers 256
GEMINI_API_KEY=... uv run --extra pipeline python scripts/tldr_data.py --max-workers 256 \
    --drift-snapshot gemini-3.5-flash --prompt-version v1-summ
GEMINI_API_KEY=... uv run --extra pipeline python scripts/tldr_data.py --max-workers 256 \
    --drift-snapshot gemini-3.1-pro-preview --prompt-version v2-strict-summ
# experiments (compute-only):
E="uv run python scripts/c2_anchor_experiment.py --dataset tldr"
# confusion (matched + under-provisioned):
$E --seeds 60 --k 200 --anchor-rate 5 --main-mode fixed --main-budget 0.3
$E --seeds 60 --k 50 --anchor-rate 20 --main-mode fixed --main-budget 0.3
# lenient bump: anchor-budget sweep (escalation main) + featured fixed@0.3:
$E --seeds 60 --k 50 --anchor-rate 20 --real-drift-snapshot gemini-3.5-flash
$E --seeds 60 --k 150 --anchor-rate 8 --real-drift-snapshot gemini-3.5-flash
$E --seeds 60 --k 200 --anchor-rate 5 --real-drift-snapshot gemini-3.5-flash
$E --seeds 60 --k 200 --anchor-rate 5 --main-mode fixed --main-budget 0.3 \
    --real-drift-snapshot gemini-3.5-flash
# harsh drift: per-axis study + guard sweep:
uv run python scripts/study_judge_version_gap.py --dataset tldr
uv run python scripts/study_judge_version_gap.py --dataset tldr \
    --snapshot gemini-3.1-pro-preview --prompt-version v2-strict
for W in 50 150 300; do \
  $E --seeds 60 --k 200 --anchor-rate 5 --guard $W --main-mode fixed \
      --main-budget 0.3 --real-drift-snapshot gemini-3.1-pro-preview \
      --drift-prompt-version v2-strict; done
# co-provisioning frontier, every cell, and baselines:
$E --seeds 30 --k 200 --guard 50 --frontier \
    --real-drift-snapshot gemini-3.1-pro-preview --drift-prompt-version v2-strict
$E --seeds 20 --k 200 --anchor-rate 5 --main-mode fixed --main-budget 0.3 --all-cells
uv run python scripts/c2_baseline_comparison.py --dataset tldr --seeds 60
\end{verbatim}}

All experiment outputs land in gitignored JSON under \texttt{artifacts/} (e.g.\
\texttt{c2\_confusion.json}, \texttt{c2\_frontier.json}, \texttt{c2\_cells.json},
\texttt{c2\_baselines.json}); the C1 cost frontier is reproduced by
\texttt{uv run python scripts/c1\_stratified\_experiment.py -{}-seeds 60} and the ten-method
comparison by \texttt{uv run python scripts/c1\_method\_comparison.py -{}-seeds 60}.

\end{document}